\pdfoutput=1
\documentclass[conference]{IEEEtran}
\IEEEoverridecommandlockouts

\usepackage{amsmath,graphicx,epstopdf,amssymb,amsthm}% ,epstopdf,spconf,amssymb,epsfig,fullpage}%, ,graphicx,psfrag,url,amsmath,amsthm,comment} 
\usepackage{cite}
\usepackage{epstopdf}
\usepackage{amsmath,bm}
\usepackage{verbatim}
\usepackage[utf8]{inputenc}
\usepackage[english]{babel}
\usepackage{subcaption}
\usepackage{caption}
\usepackage{enumitem}
\usepackage[ruled,vlined,linesnumbered]{algorithm2e}
\DeclareCaptionLabelFormat{lc}{\MakeLowercase{#1}~#2}
\usepackage[dvipsnames]{xcolor}

\newenvironment{skproof}{\noindent\textit{Sketch of  Proof:}}{\hfill$\blacksquare$}

\newtheorem{theorem}{Theorem}
\newtheorem{lemma}{Lemma}
\newtheorem{fact}{Fact}
\newtheorem{definition}{Definition}
\newtheorem{proposition}{Proposition}
\newtheorem{corollary}{Corollary}

\newtheorem{remark}{Remark}

\newtheorem{assumption}{Assumption}

\allowdisplaybreaks
\usepackage[protrusion=true,expansion=true]{microtype}

\usepackage[font=small]{caption}

\usepackage{mathtools}

\newcommand{\addFL}[1]{\textcolor{blue}{#1}}
\DeclarePairedDelimiter\floor{\lfloor}{\rfloor}

\newcommand{\nm}[1]{{\color{red} {\bf [NM: #1]}}}
\newcommand{\add}[1]{\textcolor{red}{#1}}

\usepackage{cancel}
   
\usepackage{soul,xcolor}
\DeclareMathOperator*{\argmin}{arg\,min}
\usepackage{bbm, dsfont}
 
\def\BibTeX{{\rm B\kern-.05em{\sc i\kern-.025em b}\kern-.08em
    T\kern-.1667em\lower.7ex\hbox{E}\kern-.125emX}}
\begin{document}
\bstctlcite{IEEEexample:BSTcontrol}
\setulcolor{red}
\setul{red}{2pt}
\setstcolor{red}
%\newcommand\semiSmall{\fontsize{23.8}{20.38}\selectfont}
%\title{D2D-assisted Federated Learning: Hybrid Distributed Machine Learning in Two Timescales}
% Federated Learning Beyond the Star: A Local/Global Aggregation Hybrid
% \title{Federated Learning Beyond the Star: Complementing Global Aggregations with D2D-enabled Consensus} 

%\title{Federated Learning Beyond the Star via D2D Local Learning and Global Cluster Sampling}

\title{Federated Learning Beyond the Star: Local D2D Model Consensus with Global Cluster Sampling}

% \title{Federated Learning Beyond the Star: A D2D-based Approach with Periodic Aggregations}

% hybrid approach via D2D and
% \title{Cooperative Cluster-based Federated Learning with D2D Local Model Aggregations}

% Two Timescale Hybrid Federated Learning with Cooperative D2D Local Model Aggregations

% Distributed model training under cooperative D2D communications.. 

% A consensus-driven distributed model training platform via cooperative device-to-device communications

% Some keywords: Device-to-device, peer-to-peer, Non-i.i.d data (to emphasize our new definition...), resource constrained, local descent method (we have multiple local descents),
% D2D, P2P, cluster-based, local cooperation.., locally cooperative devices ....

%  D2D-assisted hybrid federated learning with Aperiodic Consensus

\author{Frank Po-Chen Lin$^*$, Seyyedali~Hosseinalipour$^*$, Sheikh Shams Azam$^*$,\\ Christopher G. Brinton$^*$, and Nicol\`o Michelusi$^\dagger$
\\
$^*$School of Electrical and Computer Engineering, Purdue University, IN, USA\\$^\dagger$School of Electrical, Computer and Energy Engineering, Arizona State University, AZ, USA \\
\vspace{-1mm}
\thanks{An extended version of this paper is published in IEEE JSAC~\cite{lin2021timescale}. The work of N. Michelusi was supported in part by NSF under grants CNS-1642982 and CNS-2129015. C. Brinton and F. Lin were supported in part by ONR under grant N00014-21-1-2472.}
% \thanks{N. Michelusi is with the School of Electrical, Computer and Energy Engineering, Arizona State University, AZ, USA. e-mail: nicolo.michelusi@asu.edu. Part of his research has been funded by NSF under grant CNS-1642982.}

}
\maketitle

\begin{abstract} Federated learning has emerged as a popular technique for distributing model training across the network edge. Its learning architecture is conventionally a star topology between the devices and a central server. In this paper, we propose \textit{two timescale hybrid federated learning} ({\tt TT-HF}), which migrates to a more distributed topology via device-to-device (D2D) communications. In {\tt TT-HF}, local model training occurs at devices via successive gradient iterations, and the synchronization process occurs at two timescales: (i) macro-scale, where global aggregations are carried out via device-server interactions, and (ii) micro-scale, where local aggregations are carried out via D2D cooperative consensus formation in different device clusters. Our theoretical analysis reveals how device, cluster, and network-level parameters affect the convergence of {\tt TT-HF}, and leads to a set of conditions under which a convergence rate of $\small \mathcal{O}(1/t)$ is guaranteed. Experimental results demonstrate the improvements in convergence and utilization that can be obtained by {\tt TT-HF} over state-of-the-art federated learning baselines.

\end{abstract}
% \begin{IEEEkeywords}
% Federated learning, device-to-device (D2D) communications, fog learning, edge intelligence, peer-to-peer (P2P) learning, cooperative consensus formation.
% \end{IEEEkeywords}

\section{Introduction}
% \noindent 
% Machine learning (ML) techniques have demonstrated notable successes a variety of complicated tasks, e.g., image segmentation, and natural language processing.

% Machine learning techniques have traditionally relied on model training in a central location where data and computation resources coexist, e.g., a cloud server. However, in many contemporary applications, the data required for model training is generated at devices distributed across the edge of communication networks~\cite{}.
% Centralized ML is impractical in many of these settings due to (i) prohibitive time and power consumption needed to transfer the data from the devices to a central server, and (ii) privacy concerns associated with sharing raw data~\cite{hosseinalipour2020federated}. 

% Federated learning is recently proposed as an effective way of distributing machine learning (ML) model training across wireless networks~\cite{mcmahan2017communication}. It conducts ML model training by solely   requires model parameter passing among the edge devices and the main server which  eliminates prohibitive time and power consumption needed to transfer the data from the devices to a central server, and privacy concerns associated with sharing raw.

\noindent Efforts to distribute machine learning (ML) model training across contemporary networks have recently focused on federated learning. The conventional architecture that has been proposed in federated learning, depicted in Fig.~\ref{fig:simpleFL}, is a star topology between edge devices and a server. Operationally, there are two steps repeated in sequence: (i) \textit{local updating/learning}, where devices train their local models on their own datasets, often using stochastic gradient descent (SGD); and (ii) \textit{global aggregation}, where a server aggregates the local models into a global model, and synchronizes the devices.

By avoiding raw data transfers over a network, federated learning results in energy, delay, and bandwidth savings, and also alleviates privacy concerns~\cite{frank2020delay,azam2020towards}. However, implementing it over large-scale networks composed of many devices still faces two key challenges: (i) 
% upon having resource constrained devices (e.g., wireless sensors and cellular phones), 
sequential uplink transmissions can require prohibitive energy consumption for devices, and (ii) statistical heterogeneity across local device datasets can hinder convergence speed and resulting global model accuracy~\cite{hosseinalipour2020federated}.

We propose addressing these challenges by augmenting federated learning with a third step: \textit{local model aggregations} within local clusters of devices. Facilitated by device-to-device (D2D) communications in 5G-and-beyond wireless~\cite{hosseinalipour2020federated}, we will show how such local aggregations can address statistical heterogeneity while reducing communication resource utilization.

\begin{figure}[t]
\centering
\includegraphics[width=3.5in, height=.55in]{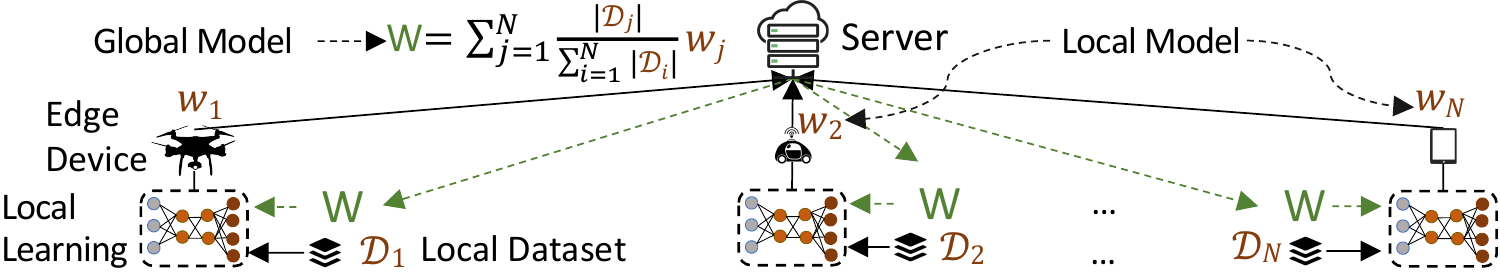}
% \centering
\caption{Architecture of conventional federated learning.}
% , which is broadcast to the devices for the next round of local updates.}
\label{fig:simpleFL}
\vspace{-6mm}
\end{figure}

\textbf{Related Work:} There have been a multitude of works on optimizing federated learning in recent years (see e.g.,~\cite{rahman2020survey} for a survey). One of the main techniques for addressing variations in communication resource constraints has been reducing/adapting the frequency of global aggregations, allowing devices to further refine models locally in-between aggregations~\cite{wang2019adaptive,tu2020network,frank2020delay}. Hierarchical system models have also been proposed to further reduce the demand for global aggregations~\cite{liu2020client}. However, the degree of device dataset heterogeneity will impact the extent to which the global aggregation frequency can be reduced without local model overfitting~\cite{wang2019adaptive}. There have also been works on improving model training in the presence of heterogeneous data among the devices~\cite{9149323,zhao2018federated}, typically through data sharing facilitated by the server. However, raw data sharing may suffer from privacy concerns or bandwidth limitations.

% By allowing devices to perform multiple local gradient iterations in-between aggregations, they can further refine their local models while decreasing the frequency of uplink transmissions.
%owever, upon having extreme data heterogeneity across the devices, the local models would become biased toward local datasets, which can severely affect the performance of the global model and degrade the convergence speed~\cite{wang2019adaptive}. Thus, the main motivation of this work is to mitigate this issue by proposing a cooperative cluster based learning methodology. 

The literature on federated learning has mainly focused on the star learning topology depicted in Fig.~\ref{fig:simpleFL}. This has motivated a set of recent works on fully decentralized (server-less) federated learning~\cite{8950073,9154332}. Our work proposes an intermediate between the star topology and fully distributed learning for settings where a server is available. In this respect, our work is most closely related to~\cite{hosseinalipour2020multi}, which augments federated learning with peer-to-peer (P2P) interactions among local devices. Different from~\cite{hosseinalipour2020multi}, we consider the scenario where multiple local SGD iterations are conducted in-between aggregations and where consensus occurs aperiodically, leading to a more complex learning model that we analyze.

\textbf{Summary of Contributions:} We develop \textit{two timescale hybrid federated learning} ({\tt TT-HF}), a novel methodology for improving distributed model training efficiency by blending federated aggregations with cooperative consensus formation among local device clusters at a shorter timescale (Sec. \ref{sec:tthf}). We theoretically analyze the convergence behavior of {\tt TT-HF} and obtain a set of conditions on its parameters for which converge at a rate of $\small \mathcal{O}(1/t)$ is guaranteed (Sec. \ref{sec:convAnalysis}). Our experiments verify that {\tt TT-HF} can outperform federated learning baselines substantially in terms of model and resource metrics (Sec. \ref{sec:experiments}).

\section{System Model and Learning Architecture}
\label{sec:tthf}
% \noindent In this section, we first describe our edge network system model of D2D-enabled clusters (Sec.~\ref{subsec:syst1}) and formalize the ML task for the system (Sec.~\ref{subsec:syst2}). Then, we develop our two timescale hybrid federated learning algorithm, {\tt TT-HF} (Sec.~\ref{subsec:syst3}).
% \addFL{For the following descriptions, we omit the superscript $(k)$ without causing confusion, but keep in mind that XXX may vary with $k$.}

\subsection{Edge Network Model}
\label{subsec:syst1}
We consider ML model training over the network architecture depicted in Fig.~\ref{fig2}. It consists of an edge server and $I$ edge devices gathered via the set $\mathcal{I}=\{1,\cdots,I\}$.
% We consider model learning over the network architecture depicted in Fig.~\ref{fig2}. The network consists of an edge server and $I$ edge devices gathered by the set $\mathcal{I}=\{1,\cdots,I\}$. 
% \nm{$\mathcal{I}$ appears to be unused..}
We partition the edge devices into $N$ sets of clusters denoted by $\mathcal{S}_1,\cdots,\mathcal{S}_N$, where $\mathcal{S}_c \cap \mathcal{S}_{c'} =\emptyset~\forall c\neq c'$ with $\cup_{c=1}^{N} \mathcal{S}_{c}=\mathcal{I}$.
Cluster $\mathcal{S}_c$ contains $s_c = |\mathcal{S}_c|$ edge devices capable of performing D2D communications with their neighbors.
% Inside cluster $\mathcal{S}_c$, there exist a set of $s_c\triangleq |\mathcal{S}_c|$ edge devices, which are capable of performing D2D communications.\footnote{Without causing ambiguity, $\mathcal{S}_c$ is used to refer to the cluster and the set of nodes inside it.}
% We consider a \textit{cluster-based representation} of the edge, where the devices are partitioned into $N$ sets $\mathcal{S}_1,\cdots,\mathcal{S}_N$, $\mathcal{S}_c \cap \mathcal{S}_{c'} =\emptyset~\forall c\neq c'$ and $\cup_{c=1}^{N} \mathcal{S}_{c}=\mathcal{I}$. Cluster $\mathcal{S}_c$ contains $s_c = |\mathcal{S}_c|$ edge devices, where $\sum_{c=1}^{N}s_c=I$.
%For convenience, $\mathcal{S}_c$ will be used to refer to both the cluster itself and to the set of nodes inside it. 
% We assume that the edge devices in a cluster participate in D2D communications with its neighbors
% We assume that the clusters are formed based on the reachabli of devices to conduct low-energy D2D communications, e.g., geographic proximity. 
% Thus, one cluster may be a fleet of drones while another is a collection of local IoT sensors. In general, we do not place any restrictions on the composition of devices within a cluster, as long as they possess a common D2D protocol~\cite{hosseinalipour2020federated} and communicate with a common server.
% We let $\mathcal{N}_i\subseteq S_c$ denote the set of neighbors for edge device $i\in \mathcal S_c$.
The D2D communications between neighbors are bidirectional, i.e., $i \in \mathcal{N}_{i'}$ if and only if ${i'}\in{\mathcal N}_i$, $\forall i,i'\in\mathcal{S}_c$, where $\mathcal{N}_i\subset S_c$ denotes the set of neighbors of edge device $i\in \mathcal S_c$.
Consequently, we associate a network graph to each cluster $\mathcal{S}_c$ denoted by $G_c(\mathcal{S}_c, \mathcal E_c)$, where $\mathcal{S}_c$ denotes the set of nodes and $\mathcal E_c$ denotes the set of edges, where $(i,{i'})\in \mathcal E_c$ if and only if $i,i'\in\mathcal{S}_c$ and $i \in \mathcal {N}_{i'}$.

\begin{figure}[t]
%   \centering
%   \vspace{-0.4in}
    \includegraphics[height=30mm,width=88mm]{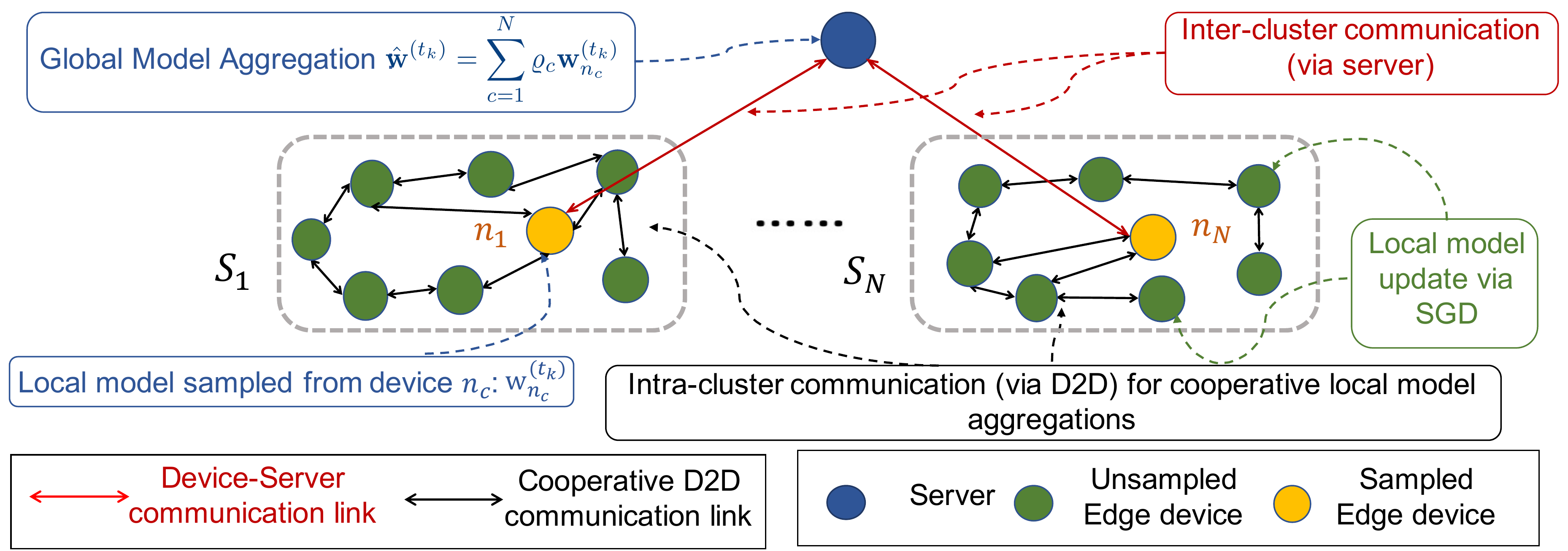}
     \caption{
    %  \nm{figure is too dense. Maybe show only 2 clusters since it captures the essence of the model? You can write in the caption "example with two clusters"}
     Network architecture of D2D-assisted federated learning. Devices cooperatively form cluster topologies using D2D communications, through which they perform consensus on their local models.}
    %  Cooperative local model aggregations among these clusters occur in between global aggregations by the server.}
     \label{fig2}
     \vspace{-0.7em}
\end{figure}
 
\subsection{Machine Learning Task} \label{subsec:syst2}
% \subsubsection{Training dataset}
Device $i$ has a local dataset $\mathcal{D}_i$ consisting of $D_i=|\mathcal{D}_i|$ data points. 
% In general, the number of data points can vary from one device to another. 
Each data point $(\mathbf x,y)\in\mathcal D_i$ consists of an $m$-dimensional feature vector $\mathbf x\in\mathbb R^m$ and a label $y\in \mathbb{R}$. 
% Let $D_i=|\mathcal{D}_i|$ denote the number of data points 
% \subsubsection{Loss functions and ML objective}
Letting $\hat{f}(\mathbf x,y;\mathbf w)$ denote the loss based on the learning model parameter vector $\mathbf w \in \mathbb{R}^M$ associated with $(\mathbf x,y)\in\mathcal D_i$, the local loss function for device $i$ is defined as:
% For example, in linear regression, $\hat{f}(\mathbf x,y;\mathbf w) = \frac{1}{2}(y-\mathbf w^\top\mathbf x)^2$. The \textit{local loss function} at device $i$ is defined as
\begin{align}\label{eq:1}
    F_i(\mathbf w)=\frac{1}{D_i}\sum_{(\mathbf x,y)\in\mathcal D_i}
    \hat{f}(\mathbf x,y;\mathbf w).
\end{align}
% \chris{Should point out here that we make no assumptions on $\mathcal{D}_i$ except that they are a subset of the global dataset.} 
We define the cluster loss function for $\mathcal{S}_c$ as:
\begin{align}\label{eq:c}
    \hat F_c(\mathbf w)=\sum_{i\in\mathcal{S}_c} \rho_{i,c} F_i(\mathbf w),
\end{align}
% \nm{there is a problem here.. since clusters are time varying, $\hat F_c(\mathbf w)$ is TV as well! Is our analysis able to handle this scenario?}
where $\rho_{i,c}= 1/s_c$ is the weight associated with device $i\in \mathcal{S}_c$. 
% \chris{This definition of $\rho$ is weighting each device the same. Can we generalize to the case where we weight by number of datapoints? It's more realistic.} 
% \chris{I agree with Nicolo, per my previous comment. If we want to motivate with heterogeneity it would be better to handle the case where the number of datapoints are not the same at each device. I suspect that our analysis holds even for a general $\rho_{i,c}^{(k)}$.}
The global loss function across the network is then given by:
\begin{align} \label{eq:2}
    F(\mathbf w)=\sum_{c=1}^{N} \varrho_c \hat F_c(\mathbf w),
\end{align}
where $\varrho_c= s_c/ I$ is the weight associated with cluster $\mathcal{S}_c$ relative to the network. 
% In this work, we will set $\rho_{i,c}= 1/s_c$, i.e., weighting each cluster's nodes equally, and $\varrho_c= s_c \big(\sum_{i=1}^N s_i \big)^{-1}$, i.e., weighting each cluster according to its size.
% The weight of each edge node $i \in\mathcal{S}_c$ relative to the network can thus be obtained as $\rho_i=\varrho_c\cdot\rho_{i,c}= 1 / I$, meaning each node has an equal weight in $F$.
% \chris{Following on the previous comment, I think we can write $\rho_{i,c}^{(k)} = D_i (\sum_{j \in \mathcal{S}_c^{(k)}} D_j)^{-1}$ and $\varrho_c^{(k)} = (\sum_{j \in \mathcal{S}_c^{(k)}} D_j) (\sum_{i \in \mathcal{I}} D_i)^{-1}$.}
% \nm{why is $\rho_i$ not a function of k?}
The ultimate goal of the ML model task is to find the optimal parameter vector $\mathbf w^*\in \mathbb{R}^M$ that minimizes the global loss function: % \nm{you have already defined this as $m$ before!}
$
    \mathbf w^* = \mathop{\argmin_{\mathbf w \in \mathbb{R}^M} }F(\mathbf w).
$
\begin{assumption}\label{Assump:SmoothStrong}
\label{beta}
We make the following standard assumptions~\cite{chen2019joint} on the global and local loss functions: 
\begin{itemize}[leftmargin=3mm]
\item  \textbf{Strong convexity}:
% \nm{define $\mu\beta$ as the strong convexity param so that $\mu\in[0,1]$..}
 $F$ is $\mu$-strongly convex, i.e.,
%  \footnote{Convex ML loss functions, e.g., squared SVM and linear regression, are implemented with a regularization term in practice to improve convergence and avoid model overfitting, which makes them strongly convex.}
\begin{equation} \label{eq:11} 
    \hspace{-4.5mm}\resizebox{.98\linewidth}{!}{$
    F(\mathbf w_1) \geq  F(\mathbf w_2)+\nabla F(\mathbf w_2)^\top(\mathbf w_1-\mathbf w_2)
    % \nonumber \\&
    % ~~~~
    +\frac{\mu}{2}\big\Vert\mathbf w_1-\mathbf w_2\big\Vert^2 \hspace{-1.2mm},~\hspace{-.5mm}\forall { \mathbf w_1,\mathbf w_2}.
    $}\hspace{-2.mm}
\end{equation}
    \item  \textbf{Smoothness:} $F_i$ is $\beta$-smooth $\forall i$, i.e., 
    \begin{align}
\big\Vert \nabla F_i(\mathbf w_1)-\nabla F_i(\mathbf w_2)\big\Vert \leq & \beta\big\Vert \mathbf w_1-\mathbf w_2 \big\Vert,~\forall i, \mathbf w_1, \mathbf w_2.
 \end{align}
%  , i.e., $\Vert\nabla F(\mathbf w)\Vert^2 \leq 2\beta[F(\mathbf w)-F(\mathbf w^*)]$, $\forall \mathbf w$.\nm{why do you call this inequality "$\beta$-smoothness of the global function,"?}
\end{itemize}
% \begin{description}
% \item[$\beta$-smoothness:] 

% \nm{are you assuming that $F_i$ are convex? You need to state that..}
% \frank{I am not assuming $F_i$ to be convex}
% \end{description}
\end{assumption}
% We further define $\vartheta = \mu/\beta < 1$.\nm{no dont add this, keep notation simple..with my def above, $\mu$ is already the ratio} 
% While we leverage these assumptions in our theoretical development, our experiments in Appendix H of~\cite{techReport} demonstrate that our resulting methodology is still effective in the case of non-convex loss functions (in particular, for neural networks).

% \begin{assumption} \label{PL}
% $F(\cdot)$ satisfies the Polyak-≈Åojasiewicz (PL) condition with constant $\mu$.
% \begin{align} \label{eq:11}
%     \frac{1}{2}\Big\Vert\nabla F(\mathbf w)\Big\Vert^2 \geq \mu(F(\mathbf w)-F(\mathbf w^*)),\ \forall \mathbf w,
% \end{align}
% \end{assumption} 

% \begin{assumption} \label{gradVar}
%     $$\mathbb E[\Vert(\widehat{\mathbf g}_j(\mathbf w_j(t))-\nabla F_j(\mathbf w_j(t))\Vert_2^2]\leq\sigma^2$$
% \end{assumption}
We measure the statistical heterogeneity across the local datasets using the following gradient diversity metric:
\begin{definition}[Gradient Diversity]\label{gradDiv}
The gradient diversity across the device clusters is measured via $\delta\in\mathbb{R}^+$ that satisfies 
\begin{align} 
    \big\Vert\nabla\hat F_c(\mathbf w)-\nabla F(\mathbf w)\big\Vert
    \leq \delta,~\forall c, \mathbf w.
\end{align}
\end{definition} 

\begin{figure}[t]
\includegraphics[width=0.51\textwidth]{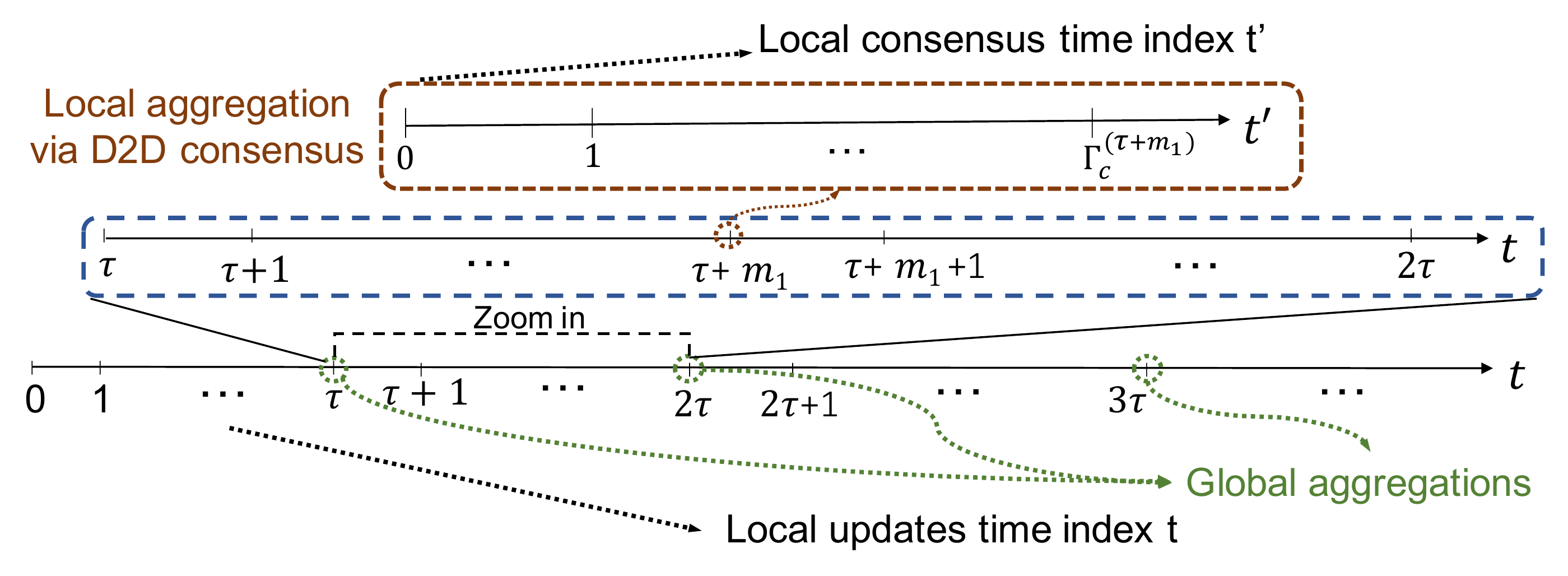}
\centering
\caption{Depiction of timescales in {\tt TT-HF}. Time index $t$ captures the local descent iterations and global aggregations. Time index $t'$ captures the rounds of local D2D communications/consensus.} 
% In each local training interval, devices engage in consensus formation aperiodically.}
\label{fig:twoTimeScale}
\vspace{-5mm}
\end{figure}

\vspace{-2mm}
\subsection{{\tt TT-HF}: Two Timescale Hybrid Federated Learning} \label{subsec:syst3}
% \subsubsection{Overview and rationale}
{\tt TT-HF} performs model training over a sequence of global aggregations. During each global aggregation interval, device $i\in \mathcal{S}_{c}$ performs successive local SGD iterations and aperiodically engages in local model aggregations within $\mathcal{S}_{c}$.

% There are three main practical reasons for incorporating the local consensus procedure into the learning paradigm. First, local consensus can help further suppress any bias of device models to their local datasets, which is one of the main challenges faced in federated learning in environments where data may be non-i.i.d. across the network. Second, an iteration of local consensus can be expected to incur much lower device power consumption compared with the global aggregations, which require uplink transmissions to potentially far away aggregation points (e.g., from smartphone to base station) while D2D communications are performed over short ranges~\cite{hmila2019energy,dominic2020joint}. Third, D2D is becoming a prevalent feature of 5G-and-beyond wireless networks~\cite{7994918,9048621}.
% This leads to less dependency on energy-intensive global aggregations (due to the requirement of uplink transmissions) for parameter synchronization, where the global aggregations can occur less frequently. 
% \nm{I got lost with the explanations in this paragraph}

% \subsubsection{{\tt TT-HF} procedure}
% \nm{it is essential that this is explained clearly. I think it needs to be improved. The consensus is explained too late in the paragraph. I think you should follow a more logical order: SGD --> consensus --> synchronization --> repeat}

Formally, we divide the learning process into a set of discrete time indices $t \in \mathcal{T}$, where $\mathcal{T} = \{0, 1, ...\}$. At each time instance $t$, each device $i$ has a local model denoted by $\mathbf w_i^{(t)}$ and there exists a global model at the server denoted by $\hat{\mathbf{w}}^{(t)}$. The ML model training starts at $t = 0$, where the local models are initialized with $\mathbf w_i^{(0)} = \hat{\mathbf{w}}^{(0)}$ broadcast by the server.

Let $t_k \in \mathcal{T}$ denote the occurrence of the $k$th global aggregation and $\mathcal{T}_k = \{t_{k-1} + 1,...,t_k\}$ denote the $k$th local model training interval, with $\tau = |\mathcal{T}_k|$ as the length of each interval $k$. We have $t_k = k\tau$ and $t_0 = 0$.
At the $k$th global aggregation, the main server samples one device $n_c \in \mathcal{S}_c$ from each cluster $c$, and updates the global model as follows:
% The $k$th global model computed via the aggregation of local model parameters at the server is expressed by:
% Since global aggregations are aperiodic, in general $\tau_{k} \neq \tau_{k'}$ for $k \neq k'$.
%As in federated learning, there will be $K$ \textit{global aggregations}, which we will index as $k = 1,...,K$. 
% The global model computed by the server at the $k$th global aggregation is denoted as 
% The $k$th global aggregation of model parameters computed at the server will be denoted as 
\begin{align} \label{15}
    \hat{\mathbf w}^{(t)} &= 
          \sum\limits_{c=1}^N \varrho_c \mathbf w_{n_c}^{(t)}, \;\; t=t_k, ~k=1,2,...
\end{align}
% which is then broadcast by the main server to all of the edge devices to  override their local models at time $t_k$: $\mathbf w_i^{(t_{k})} = \hat{\mathbf{w}}^{(t_{k})}$, $\forall i$.
% \nm{there is no delay to wait for these new models?} \frank{the devices are directly sync with the new models}
% The process then repeats for $\mathcal{T}_{k+1}$. 
% The model training procedure starts with the server broadcasting a global model $\hat {\mathbf w}^{(0)}$ to initialize the devices' local models at $t = 0$. 
This cluster sampling results in bandwidth and power savings~\cite{hosseinalipour2020federated}. At $t=t_k$, the global model $\hat{\mathbf{w}}^{(t_{k})}$ is broadcast by the server to all devices to override their local models.
% The model training procedure starts with the server broadcasting $\hat {\mathbf w}^{(0)}$ to initialize the devices' local models.

\textbf{Local SGD iterations}: 
For $t \in \mathcal{T}_k$, each device $i \in \mathcal{I}$ updates its local model $\mathbf w_i^{(t-1)} \in \mathbb{R}^M$ by randomly sampling a mini-batch $\xi_i^{(t-1)}$ of fixed size from its local dataset $\mathcal D_i$ to calculate the unbiased local gradient estimate as:
\begin{align}\label{eq:SGD}
    \widehat{\mathbf g}_{i}^{(t-1)}=\frac{1}{\vert\xi_i^{(t-1)}\vert}\sum_{(\mathbf x,y)\in\xi_i^{(t-1)}}
    \hat{f}(\mathbf x,y;\mathbf w_i^{(t-1)}).
\end{align}
% which is an unbiased estimate of the local gradient, i.e., $\mathbb E_{t}[\widehat{\mathbf g}_{i}^{(t)}]=\nabla F_i(\mathbf w_i^{(t)})$\nm{what about its variance?}. 
Using this gradient estimate, each device then computes its intermediate updated local model as:
\begin{align} \label{8}
    {\widetilde{\mathbf{w}}}_i^{(t)} = 
           \mathbf w_i^{(t-1)}-\eta_{t-1} \widehat{\mathbf g}_{i}^{(t-1)},~t\in\mathcal T_k,
\end{align}
where $\eta_{t-1} > 0$ denotes the step size. 
Based on ${\widetilde{\mathbf{w}}}_i^{(t)}$, the updated local model $\mathbf{w}_i^{(t)}$ is computed either through setting it to ${\widetilde{\mathbf{w}}}_i^{(t)}$ or through consensus, as will be explained next.

\textbf{Local model update}: 
%\nm{simplify notation: you can say that the discussion refers to a reference round k, so you drop the dependence on k but should kkep in mind that v may vary with k..}
% At each time $t \in \mathcal{T}_k$, each cluster may engage in local consensus formation for model updating. The decision of whether to engage in this consensus process at time $t$ -- and if so, how many iterations of this process to run -- will be developed in Sec.~\ref{Sec:controlAlg} based on a performance-efficiency tradeoff optimization. 
% If the devices do not execute consensus formation, we have the conventional model update rule ${{\mathbf{w}}}_i^{(t)} = {\widetilde{\mathbf{w}}}_i^{(t)}$ from \eqref{8}. If they do, then multiple \textit{rounds} of D2D communication take place, where in each round parameter transfers occur between neighboring devices. In particular, 
If cluster $c$ engages in consensus at time $t$, the nodes conduct $\Gamma^{(t)}_{{c}}$ rounds of D2D communications.
Letting $t'$ index the rounds, each node $i \in \mathcal{S}_c$ carries out the following for $t'=0,...,\Gamma^{(t)}_{{c}}-1$:
%During each local model training interval, the devices aperiodically form consensus through engaging in cooperative D2D communications, determining the time of occurrence of which and the number of D2D communications performed are parts of design elaborated in Sec.~\ref{Sec:controlAlg}. If at time $t\in\mathcal T_k$ the devices do not form consensus, we have the conventional model update rule ${{\mathbf{w}}}_i^{(t)} = {\widetilde{\mathbf{w}}}_i^{(t)}=
%           \mathbf w_i^{(t-1)}-\eta^{(k)} \widehat{\mathbf g}_{i}^{(t-1)}$. Upon forming consensus at time $t$, the devices perform multiple \textit{rounds} of D2D,
        %   \nm{does this happen within t? Why are the consensus updates not counted at the same timescale as local SGD? How do you make sure that the clusters are synced?} 
%        each of which consists of parameter transfers between neighboring devices. In particular, each node $i\in\mathcal{S}_c$ conducts the following iteration for $t'=0,\cdots,\Gamma^{(t)}_{{c}}-1$, where $\Gamma^{(t)}_{{c}}$ denotes the rounds of D2D in the respective cluster:
    \begin{equation}\label{eq:ConsCenter}
       \textbf{z}_{i}^{(t'+1)}= v_{i,i} \textbf{z}_{i}^{(t')}+ \sum_{j\in \mathcal{N}_i} v_{i,j}\textbf{z}_{j}^{(t')},
  \end{equation}

\vspace{-0.1in}
\noindent 
where $\textbf{z}_{i}^{(0)}=\widetilde{\mathbf{w}}_i^{(t)}$ is the node's intermediate local model from \eqref{8}, and $v_{i,j}\geq 0$, $\forall i,j$ is the consensus weight that node $i$ applies to the vector received from $j$. At the end of this process, node $i$ takes $\mathbf w_i^{(t)} = \textbf{z}_{i}^{(\Gamma^{(t)}_{{c}})}$ as its updated local model. If nodes in cluster $c$ do not engage in consensus, $\mathbf w_i^{(t)}=\widetilde{\mathbf{w}}_i^{(t)}$, $\forall i \in \mathcal{S}_c$.   
%\nm{you need to make an effor to simplify notation.. these nested superscripts are really unappealing..}

\begin{algorithm}[t]
{\footnotesize
\SetAlgoLined
\caption{Two timescale hybrid federated learning} \label{TT-HF}
% \KwResult{Write here the result }
\KwIn{Number of global aggregations $K$, D2D rounds $\{\Gamma_c^{(t)}\}_{t=1}^{T},~\forall c$, length of local model training interval $\tau$} 
\KwOut{Final global model $\hat{\mathbf w}^{(T)}$}
 // Initialization at the server \\
 Initialize $\hat{\mathbf w}^{(0)}$ and broadcast it along with the indices of the sampled devices ($n_c$, $\forall c$) for the first global aggregation.\\
 \For{$k=1:K$}{
     \For{$t=t_{k-1}+1:t_k$}{
        \For{$c=1:N$}
        {
         // Procedure at the clusters \\
         Each device $i\in\mathcal{S}_c$ performs local SGD update based on~\eqref{eq:SGD} and~\eqref{8} using $\mathbf w_i^{(t-1)}$ to obtain~$\widetilde{\mathbf{w}}_i^{(t)}$.\\
        %  with:\\
        %   $\widetilde{\mathbf w}_i^{(t)} = 
        %   \mathbf w_i^{(t-1)}-\eta_{t-1} \widehat{\mathbf g}_{i}^{(t-1)}$\; 
        Devices conduct $\Gamma^{(t)}_{{c}}$ rounds of D2D based on~\eqref{eq:ConsCenter}, initializing  $\textbf{z}_{i}^{(0)}=\widetilde{\mathbf{w}}_i^{(t)}$ and setting $\mathbf w_i^{(t)} = \textbf{z}_{i}^{(\Gamma^{(t)}_{{c}})}$.
        % on $\widetilde{\mathbf w}_i^{(t)}$ as:\\
        % $\textbf{z}_{i}^{(0)}=\widetilde{\mathbf{w}}_i^{(t)}$\;
        % \For{$t'=0:\Gamma^{(t)}_{{c}}-1$}{
        %     $\textbf{z}_{i}^{(t'+1)}= v^{(k)}_{i,i} \textbf{z}_{i}^{(t')}+\hspace{-2mm} \sum_{j\in \mathcal{N}^{(k)}_i} v^{(k)}_{i,j}\textbf{z}_{j}^{(t')}$
        % }
        % $\mathbf w_i^{(t)} = \textbf{z}_{i}^{(\Gamma^{(t)}_{{C}})}$\; 
      }
      \If{$t=t_k$}{
      // Procedure at the clusters \\
      Each sampled device $n_c$ sends $\mathbf w_{n_c}^{(t_k)}$ to the server.\\
      // Procedure at the server \\
    %   Estimate $\delta$\;
     Compute $\hat{\mathbf w}(t)$ using \eqref{15}, and
      broadcast it along with the indices of the sampled devices ($n_c$, $\forall c$) for the next global aggregation.
      }
 }
}
}
\end{algorithm}

The index $t'$ corresponds to the second timescale in {\tt TT-HF}, referring to the consensus process, as opposed to the index $t$ which captures the time elapsed by the local gradient iterations.
Fig.~\ref{fig:twoTimeScale} illustrates these two timescales, where at certain local iterations $t$ the consensus process $t'$ is run.

\begin{assumption}\label{assump:cons}
The consensus matrix $\mathbf{V}_{{c}}= \left[v_{i,j} \right]_{i,j\in\mathcal{S}_c}\in \mathbb{R}^{s_c \times s_c}$, $\forall c$, satisfies the following conditions~\cite{xiao2004fast}: (i) $\left(\mathbf{V}_{{c}}\right)_{m,n}=0~~\textrm{if}~~ \left({m},{n}\right) \notin \mathcal{E}_{{c}}$; (ii) $\mathbf{V}_{{c}}\textbf{1} = \textbf{1}$, where $\textbf{1}$ denotes the column vector of size $s_c$ with unit elements; (iii) $\mathbf{V}_{{c}} = {\mathbf{V}_{{c}}}^\top$; (iv)~$\rho \big(\mathbf{V}_{{c}}-\frac{\textbf{1} \textbf{1}^\top}{s_c} \big) < 1$, where $\rho(.)$ denotes the largest eigenvalue of the matrix in the argument. %where $\textbf{1}$ is the vector of 1s and \underline{$\rho(\mathbf{A})$ is the spectral radius of matrix $\mathbf{A}$}.
\end{assumption}

% The consensus formation process can be viewed as an imperfect aggregation of the models in each cluster. Specifically, 
We can write the local parameter at device $i\in \mathcal{S}_c$ as
\begin{align} \label{eq14}
    \mathbf w_i^{(t)} = \bar{\mathbf w}_c^{(t)} + \mathbf e_i^{(t)},
\end{align}
where $\bar{\mathbf w}_c^{(t)} = \sum_{i\in\mathcal{S}_c} \rho_{i,c} \tilde{\mathbf w}_i^{(t)}$ is the average
% \nm{what do you mean by the average to be perfect? ITs just the average} 
of the local models in the cluster and $\mathbf e_i^{(t)} \in \mathbb{R}^M$ denotes the consensus error caused by limited D2D rounds (i.e., $\Gamma_c^{(t)} < \infty$) among the devices. We next introduce a definition of the divergence across intermediate updated local models, which we use to derive an upper bound on the consensus error:
% \nm{since you are "hiding" the consensus steps within the time index $t$, one may argue that you are not really capturing these time constraints..}

% \begin{definition}\label{def:clustDiv}
% The divergence of intermediate updated local model parameters of devices belonging to cluster $\mathcal{S}_c$ at time $t\in\mathcal{T}_k$, denoted by $\Upsilon^{(t)}_{{c}}$, is defined as follows: \begin{equation}\label{eq:Updef}
%   \big\lvert (\widetilde{\mathbf{w}}^{(t)}_{{i}})_{z} -(\widetilde{\mathbf{w}}^{(t)}_{{i'}})_z \big\rvert \leq \Upsilon^{(t)}_{{c}} ,~ \forall {i},{i'}\in\mathcal{S}_c, 1\leq z\leq M,
% \end{equation} 
% where $(.)_z$ denotes the $z$-th element of the indexed vector.
% \end{definition}

\begin{lemma}\label{lemma:cons}
Upon performing $\Gamma^{(t)}_{{c}}$ rounds of D2D in cluster $\mathcal{S}_c$, the consensus error $\mathbf e_i^{(t)}$ is upper-bounded as follows:
\begin{equation}
   \Vert \mathbf e_i^{(t)} \Vert  \hspace{-.5mm}  \leq (\lambda_{{c}})^{\Gamma^{(t)}_{{c}}}
\sqrt{s_c}\underbrace{\max_{j,j'\in\mathcal S_c}\Vert\tilde{\mathbf w}_j^{(t)}-\tilde{\mathbf w}_{j'}^{(t)}\Vert}_{\triangleq \Upsilon^{(t)}_{{c}}},
% \mst{   \left(\lambda_{{c}}\right)^{\Gamma^{(t)}_{{c}}} s_c{\Upsilon^{(t)}_{{c}}} M},
~ \forall i\in \mathcal{S}_c. 
\end{equation}
where each $\lambda_{{c}}$ is a constant such that $1 > \lambda_{{c}} \geq \rho \big(\mathbf{V}_{{c}}-\frac{\textbf{1} \textbf{1}^\top}{s_c} \big)$.
\end{lemma}

A summary of the {\tt TT-HF} algorithm is given in Algorithm~\ref{TT-HF}.

\section{Convergence Analysis of {\tt TT-HF}} \label{sec:convAnalysis}

% \noindent In this section, we theoretically analyze the convergence behavior of {\tt TT-HF}. Our main results are presented in Sec.~\ref{ssec:convAvg} and Sec.~\ref{ssec:sublinear}. Before then, in Sec.~\ref{ssec:definitions}, we introduce some additional definitions and a key proposition for the analysis.

\subsection{Bounding the Dispersion of Models Across Clusters}
\label{ssec:definitions}
We first introduce a standard assumption on SGD noise, and then define an upper bound on the average of consensus error for the cluster:
\begin{assumption} \label{assump:SGD_noise}
    Let ${\mathbf n}_{i}^{(t)}=\widehat{\mathbf g}_{i}^{(t)}-\nabla F_i(\mathbf w_{i}^{(t)})$, $\forall i,t$ denote the SGD noise for device $i$,  $\mathbb{E}[{\mathbf n}_{i}^{(t)}]=0$. We assume a bounded  variance for the noise, where $\exists \sigma>0: \mathbb{E}[\Vert{\mathbf n}_{i}^{(t)}\Vert^2]\leq \sigma^2$, $\forall i,t$.
\end{assumption}
% \noindent  In this section, we study the convergence of
% {\tt TT-HF} in terms of the optimality gap
%  $F(\hat{\mathbf w}^{(K)})-F(\mathbf w^*)$ between the
% global objective function at the algorithm output $\hat{\mathbf w}^{(K)}$ and at the globally optimal parameter vector $\mathbf w^*$.   

\begin{definition} \label{paraDiv}
During local model training interval $\mathcal{T}_k$, we define $\epsilon_c^{(t)}$ as an upper bound on the average of the consensus errors across the nodes in cluster $c$ at $t \in \mathcal{T}_k$, $\forall k$:
    \begin{align}
        \frac{1}{s_c}\sum\limits_{i\in \mathcal S_c}\Vert \mathbf{e}_i^{(t)}\Vert^2 \leq (\epsilon_c^{(t)})^2.
    \end{align}
    We further define $(\epsilon^{(t)})^2=\sum\limits_{c=1}^{N}\rho_c(\epsilon_c^{(t)})^2$.
    % \chris{Is $\epsilon(k)$ another definition? It sounds weird to say ``where'' since it does not appear in (13).}
\end{definition}
% In Definition \ref{paraDiv}, we characterize the consensus error in each cluster $c$ by $\epsilon_c(k)$ and the average of consensus error across clusters in the network by $\epsilon(k)$. 

Considering Definition~\ref{paraDiv} together with Lemma~\ref{lemma:cons}, if we increase the rounds of consensus, $\epsilon_c^{(t)}$ and $\epsilon^{(t)}$ can take smaller values, as we would intuitively expect.

We next define model dispersion, which measures the degree to which the cluster models deviate from the global average:

\begin{definition} \label{modDisp}
The expected model dispersion across the clusters at time $t$ denoted by $A^{(t)}$ is defined as follows:
\begin{equation}\label{eq:defA}
\hspace{-3mm}\resizebox{.55\linewidth}{!}{$
\begin{aligned}
   A^{(t)} = \mathbb E\left[\sum\limits_{c=1}^N\varrho_{c}\big\Vert\bar{\mathbf w}_c^{(t)}-\bar{\mathbf w}^{(t)}\big\Vert^2\right],
\end{aligned}
$}
\end{equation}
where $\bar{\mathbf w}_c^{(t)}$ is defined in~\eqref{eq14} and $\bar{\mathbf w}^{(t)} = \sum_{c=1}^{N}\varrho_c\bar{\mathbf w}_c^{(t)}$ denotes the global average of the local models at time $t$.
\end{definition}

% $A^{(t)}$ measures the degree to which the cluster models deviate from their average throughout the training process. 
% Obtaining an upper bound on this quantity is non-trivial due to the coupling between the gradient diversity and the model parameters imposed by~\eqref{eq:11}. For an appropriate choice of step size in~\eqref{8}, we upper bound this quantity at time $t$ through a set of new techniques that include the mathematics of \textit{coupled dynamic systems}. Specifically, we have the following result:

\begin{proposition}\label{Local_disperseT}
    If {\small $\eta_t=\frac{\gamma}{t+\alpha}$, $\epsilon^{(t)}$} is non-increasing with respect to $t\in \mathcal T_k$, i.e., $\epsilon^{(t+1)}/\epsilon^{(t)} \leq 1$, and $\alpha\geq
    \gamma\beta^2/\mu$, then the following upper bound on the expected model dispersion holds for {\tt TT-HF}:
    \begin{equation} \label{eq:At}
    \hspace{-3mm}\resizebox{.75\linewidth}{!}{$
        \begin{aligned}
           \hspace{-4mm}& A^{(t)}\hspace{-.6mm}\leq\hspace{-.6mm}
             12\left(\varrho^{\mathsf{min}}\right)^{-1}(\Sigma_{t})^2
            \left[\frac{\sigma^2}{\beta^2}+\frac{\delta^2}{\beta^2}+(\epsilon^{(0)})^2\right]\hspace{-1mm}, \hspace{-1mm}~t\in\mathcal{T}_k,\hspace{-4mm}
        \end{aligned}
        $}
    \end{equation}
        where {\footnotesize $\varrho^{\mathsf{min}}=\min_c \varrho_c$ and 
        $
            \Sigma_{t}=\sum\limits_{\ell=t_{k-1}}^{t-1}\beta\eta_\ell\left(\prod_{j=\ell+1}^{t-1}(1+2\eta_j\beta)\right).
        $}
\end{proposition}
% \frank{I will change it to unscaled notation once the content is confirmed}
\begin{skproof}
    % The complete proof is contained in Appendix~\ref{app:Local_disperse} of our online technical report~\cite{techReport}. 
    % where 
    % we break down the proof into three parts. In Part I, we find the relationship between $\Vert\bar{\mathbf w}^{(t)}-\mathbf w^*\Vert$ and $\sum\limits_{c=1}^N\varrho_{c}\Vert\bar{\mathbf w}_c^{(t)}-\bar{\mathbf w}^{(t)}\Vert$, which turns out to form a coupled dynamic system that we then solve in Part II. Finally, Part III draws the connection between $A^{(t)}$ and the solution of the coupled dynamic system, which yields the bound on $A^{(t)}$. 
    % A summary of the steps is given below:
% \textit{Part I. Finding the relationship between $\Vert\bar{\mathbf w}^{(t)}-\mathbf w^*\Vert$ and $\sum\limits_{c=1}^N\varrho_{c}\Vert\bar{\mathbf w}_c^{(t)}-\bar{\mathbf w}^{(t)}\Vert$}:
Using the definition of $\bar{\mathbf w}_c^{(t+1)}$ from Assumption~\ref{assump:cons} and $\bar{\mathbf w}^{(t+1)}$ from Definition~\ref{modDisp}, we have:
    \begin{equation} \label{eq:w_cT}
       \resizebox{.78\linewidth}{!}{$
        \bar{\mathbf w}_c^{(t+1)}=\bar{\mathbf w}_c^{(t)}
        -\frac{\eta_t}{s_{c}}\sum\limits_{j\in\mathcal S_{c}}\nabla F_j(\mathbf w_j^{(t)})
        -\frac{\eta_t}{s_{c}}\sum\limits_{j\in\mathcal S_{c}}\mathbf n_j^{(t)}, 
        $} \hspace{-3mm}
    \end{equation} 
    \begin{equation} \label{eq:w-T}
    \hspace{-0.6mm}\resizebox{.9\linewidth}{!}{$
        \bar{\mathbf w}^{(t+1)}
       \hspace{-.5mm} =\hspace{-.5mm}
        \bar{\mathbf w}^{(t)}\hspace{-.5mm}
        -\hspace{-.7mm}\sum\limits_{d=1}^N\varrho_{d}\frac{\eta_t}{s_{d}}\hspace{-.9mm}\sum\limits_{j\in\mathcal S_{d}}\hspace{-.9mm}\nabla F_j(\mathbf w_j^{(t)})
        % \nonumber \\&
       \hspace{-.5mm} -\hspace{-.7mm}\sum\limits_{d=1}^N\varrho_{d}\frac{\eta_t}{s_{d}}\hspace{-.9mm}\sum\limits_{j\in\mathcal S_{d}}\hspace{-.9mm}\mathbf n_j^{(t)}\hspace{-1mm}. 
        $} \hspace{-4mm}
    \end{equation}
    Using~\eqref{eq:w_cT} and~\eqref{eq:w-T}, Assumption~\ref{beta}, Definition \ref{paraDiv}, Definition \ref{gradDiv} and Assumption~\ref{assump:SGD_noise}, and noting that $\eta_t\leq\frac{\mu}{\beta^2}$, we get the following for $t \in \mathcal{T}_k$:
     \begin{equation} \label{eq:fact_x1T}
     \hspace{-1.2mm}\resizebox{0.9\linewidth}{!}{$
     \begin{aligned} 
        &\sqrt{\mathbb E[(\sum\limits_{c=1}^N\varrho_c\Vert\bar{\mathbf w}_c^{(t+1)}-\bar{\mathbf w}^{(t+1)}\Vert)^2]} 
        \leq  
        % \\&
        \tilde{\eta}_t\left(2\sum\limits_{d=1}^N\varrho_{d}\epsilon_{d}^{(t)}+\frac{\delta}{\beta}+2 \frac{\sigma}{\beta}\right)
        \\&
        +(1+2\eta_t\beta)
        \sqrt{\mathbb E[(\sum\limits_{c=1}^N\varrho_c\Vert\bar{\mathbf w}_c^{(t)}-\bar{\mathbf w}^{(t)}\Vert)^2]}.
    \end{aligned}  
    $} \hspace{-4mm}
    \end{equation}
    Recursive expanding~\eqref{eq:fact_x1T}, taking the square of both sides and applying the Cauchy-Schwarz inequality we get:
    \begin{equation} \label{eq:recurseT}
  \resizebox{.92\linewidth}{!}{$ 
    \begin{aligned} 
        &
         \mathbb E\left[(\sum_{c=1}^N\varrho_c\Vert\bar{\mathbf w}_c^{(t)}-\bar{\mathbf w}^{(t)}\Vert)^2\right]
        \leq
        12[\Sigma_{t}]^2
        \left[\frac{\sigma^2}{\beta^2}+\frac{\delta^2}{\beta^2}+(\epsilon^{(0)})^2\right].
    \end{aligned}
        $}\hspace{-6mm}
        \end{equation}
    Thus, we obtain the following relationship:
    \begin{align} \label{eq:chi}
    A^{(t)}
    &
    \leq
    \left(\varrho^{\mathsf{min}}\right)^{-1} \mathbb E\left[\left(\sum_{c=1}^N\varrho_c\Vert\bar{\mathbf w}_c^{(t)}-\bar{\mathbf w}^{(t)}\Vert\right)^2\right].
    \end{align} 
    Finally, after combining~\eqref{eq:recurseT} and~\eqref{eq:chi}, the result of the proposition directly follows (see Appendix B).
\end{skproof}

\subsection{General Convergence Behavior of $\hat{\mathbf{w}}^{(t)}$}
\label{ssec:convAvg}
The bound in~\eqref{eq:At} demonstrates how the expected model dispersion varies in terms of the consensus error ($\epsilon^{(0)}$), the SGD noise ($\sigma^2$), and the local datasets heterogeneity ($\delta$). In the following theorem, we bound the expected one-step decrease in the global loss, as a function of model dispersion:

\begin{theorem} \label{co1}
        If $\eta_t \leq 1/\beta$ $\forall t$,
        % $\epsilon^{(t)}=\eta_t\phi,~\forallt$
  the one-step behavior of $\hat{\mathbf w}^{(t)}$ upon using {\tt TT-HF}, at $t\in\mathcal{T}_k$ is given by:
% https://www.overleaf.com/project/5f7c9b5ce460a000011dc1e1
\begin{equation} \label{eq:th1mainRes}
\hspace{-3mm}\resizebox{.93\linewidth}{!}{$
    \begin{aligned}
       &\mathbb E\left[F(\hat{\mathbf w}^{(t+1)})-F(\mathbf w^*)\right]
        \leq
        (1-\mu\eta_{t})\mathbb E[F(\hat{\mathbf w}^{(t)})-F(\mathbf w^*)]
         \\&
        +\underbrace{\frac{\eta_{t}\beta^2}{2}A^{(t)}
        +\frac{1}{2}[\eta_{t}\beta^2(\epsilon^{(t)})^2+\eta_{t}^2\beta\sigma^2+\beta(\epsilon^{(t+1)})^2]}_{(a)},
    \end{aligned}
    $}
\end{equation}
where $A^{(t)}$ is defined in~\eqref{eq:defA}.
% \nm{You have already defined this.. just say where At is defined in Def 4.}
% \begin{align}
%      A^{(t)}\triangleq\mathbb E\left[\sum\limits_{c=1}^N\varrho_{c}\Vert\bar{\mathbf w}_c^{(t)}-\bar{\mathbf w}^{(t)}\Vert_2^2\right].
% \end{align}
\end{theorem}

\begin{skproof}
% For the detailed proof, see Appendix~\ref{app:thm1}. Here, we provide a summary of the steps taken in the appendix.
% Consider $\bar{\mathbf{w}}$ as in Definition~\ref{modDisp}.
% For $t \in \mathcal T_k$, using \eqref{8}, \eqref{eq14}, and the fact that $\sum\limits_{i\in\mathcal{S}_c} \mathbf e_{i}^{{(t)}}=0$ $\forall t$, we have
% \vspace{-1mm}
% \begin{equation}\label{eq:Glob1}
%  \resizebox{.99\linewidth}{!}{$ 
%     \bar{\mathbf w}^{(t+1)} =
%     \bar{\mathbf w}^{(t)}
%     - \eta_{t}\sum\limits_{c=1}^N \varrho_c\frac{1}{s_c}\sum\limits_{j\in\mathcal{S}_c} 
%     \nabla F_j(\mathbf w_j^{(t)}) 
%     % \nonumber \\&
%     -\eta_t\sum\limits_{c=1}^N \varrho_c\frac{1}{s_c}\sum\limits_{j\in\mathcal{S}_c} 
%              \mathbf n_j^{{(t)}}
%              $}
% \end{equation}
% given Assumption~\ref{assump:cons}.
Combining the result of~\eqref{eq:w-T} with $\beta$-smoothness and applying Assumption~\ref{assump:SGD_noise}, we have:
\begin{equation} \label{eq:aveGlo1}
    \hspace{-3mm}\resizebox{.75\linewidth}{!}{$
    \begin{aligned}
    &\mathbb E_t \left[F(\bar{\mathbf w}^{(t+1)})-F(\mathbf w^*)\right]
        \leq F(\bar{\mathbf w}^{(t)}) - F(\mathbf w^*)
        \\&
        -\eta_{t}\nabla F(\bar{\mathbf w}^{(t)})^\top \sum\limits_{c=1}^N\varrho_c\frac{1}{s_c}\sum\limits_{j\in\mathcal{S}_c}\nabla F_j(\mathbf w_j^{(t)})
        \\&
        + \frac{\eta_{t}^2 \beta}{2}\Big
        \Vert\sum\limits_{c=1}^N\varrho_c\frac{1}{s_c}\sum\limits_{j\in\mathcal{S}_c} \nabla F_j(\mathbf w_j^{(t)})
        \Big\Vert^2
        +\frac{\eta_{t}^2 \beta\sigma^2}{2},
\end{aligned}
$}
\end{equation}
where $\mathbb E_t$ denotes the conditional expectation, conditioned on $\bar{\mathbf{w}}^{(t)}$.
%Using $\beta$-smoothness of the global function $F$, it can be shown that: $
    %F(\mathbf w_1) \leq  F(\mathbf w_2)+\nabla F(\mathbf w_2)^\top(\mathbf w_1-\mathbf w_2)+\frac{\beta}{2}\Big\Vert\mathbf w_1-\mathbf w_2\Big\Vert^2,~\forall { \mathbf w_1,\mathbf w_2}$.
Applying the law of total expectation and Assumption~\ref{assump:cons}, since $\eta_t\leq 1/\beta$ we have
\begin{align} \label{ld2}
    % \hspace{-4mm}
    &\mathbb E\left[F(\bar{\mathbf w}^{(t+1)})-F(\mathbf w^*)\right]
        \leq
        (1-\mu\eta_{t})\mathbb E[F(\bar{\mathbf w}^{(t)})-F(\mathbf w^*)]
        \nonumber \\&
        \qquad\qquad +\frac{\eta_{t}\beta^2}{2}A^{(t)}
        +\frac{1}{2}(\eta_{t}\beta^2(\epsilon^{(t)})^2+\eta_{t}^2\beta\sigma^2).
        % \hspace{-4mm}
\end{align}
Using the smoothness and strong convexity of $F$, we establish the relationship between $\mathbb E[F(\bar{\mathbf w}^{(t)})-F(\mathbf w^*)]$ and $\mathbb E[F(\hat{\mathbf w}^{(t)})-F(\mathbf w^*)]$ to conclude the proof (See Appendix C).
% Using the upper bound on $A^{(t)}$ from Proposition \ref{Local_disperse} in \eqref{ld2} completes the proof.
\end{skproof}
 
% Theorem \ref{co1} illustrates how the global parameter $\hat{\mathbf w}(t_k)$ converges to the optimal $\mathbf w^*$ with respect to each global update in {\tt TT-HF} and how the learning and system parameters affects this convergence. 

Theorem~\ref{co1} quantifies the one step behavior of the global model during a given local period $\mathcal{T}_k$. Considering~\eqref{eq:th1mainRes},  the convergence of sequence $\left\{\mathbb E\left[F(\hat{\mathbf w}^{(t+1)})-F(\mathbf w^*)\right]\right\}_{t=1}^{\infty}$ depends on several factors: (i) the characteristics of the loss function (i.e., $\mu, \beta$); (ii) the step size (i.e., $\eta_t$); (iii) the expected model dispersion (i.e., $A^{(t)}$, bounded by Proposition~\ref{Local_disperseT}); and (iv) the SGD noise (i.e., $\sigma^2$). 
% In fact, without careful choice of our control parameters, the sequence may diverge. Thus, we are motivated to find conditions under which convergence is guaranteed, and furthermore, under which the upper bound in~\eqref{eq:th1mainRes} will approach zero.
 
We aim for {\tt TT-HF} to match the asymptotic convergence behavior of centralized SGD under a diminishing step size, which is $\mathcal{O}(1/t)$ \cite{Bubeck}. From~\eqref{eq:th1mainRes}, to have this desired characteristic, the terms in $(a)$ should be in the order of $\mathcal{O}(\eta_t^2)$, the same as the SGD noise $\sigma^2$.
This implies that $A^{(t)} \leq \mathcal{O}(\eta_t)$ and $\epsilon^{(t)} \leq \mathcal{O}(\eta_t)$. Under the conditions expressed in Proposition~\ref{Local_disperseT}, we have $A^{(t)} \leq \mathcal{O}(\eta_t)$ (verified in the proof of Theorem~\ref{thm:subLin}). Also, since the consensus error can be controlled via the number of D2D rounds (Lemma~\ref{lemma:cons}), it would be sufficient if we choose $\epsilon^{(t)}=\eta_t\phi$, with $\phi\in\mathbb{R}^+$. %Note that this can be achieved via limited D2D communications rounds.
% Note that achieving the bound in Theorem~\ref{co1} requires $\epsilon^{(t)} = \eta_t \phi$, i.e., that the average consensus error is on the order of the step size. 
% To see that this can be achieved, note from Lemma \ref{lemma:cons} that the upper bound of $\Vert \mathbf e_i^{(t)} \Vert$ is a linear function of the divergence $\Upsilon^{(t)}_c$. 
% $\Upsilon^{(t)}_c$ in turn is on the order of $\eta_t$ according to Definition~\ref{def:clustDiv}: $\Upsilon^{(t)}_c = {\arg\max}_{z, i,i'\in\mathcal{S}_c} \lvert (\widetilde{\mathbf{w}}^{(t)}_{{i}})_{z} - (\widetilde{\mathbf{w}}^{(t)}_{{i'}})_z \big\rvert \leq \big\Vert \widetilde{\mathbf{w}}^{(t)}_{{i}} -\widetilde{\mathbf{w}}^{(t)}_{{i'}}\big\Vert \approx \eta_t \big\Vert \widehat{\mathbf g}_{i}^{(t-1)}-\widehat{\mathbf g}_{i'}^{(t-1)}\big\Vert$ from \eqref{8} if we assume the difference ${\mathbf{w}}^{(t-1)}_{{i}} - {\mathbf{w}}^{(t-1)}_{{i'}}$ in initial model parameters at $t-1$ is negligible compared to the gradients. Further, from Lemma~\ref{lemma:cons}, we can control the value of $\epsilon^{(t)}$ through the rounds of consensus performed in each cluster $c$, as it is decreasing exponentially in $\Gamma^{(t)}_c$. 
We next build upon this logic to derive a set of conditions under which the convergence rate of $\mathcal{O}(1/t)$ is achieved for $\hat{\mathbf w}^{(t)}$.

\subsection{Sublinear Convergence Rate of $\hat{\mathbf{w}}^{(t)}$}
\label{ssec:sublinear}
% Among the quantities involved in Theorem~\ref{co1}, $\eta_t, \tau$ and $\epsilon^{(t)}$ are the three tunable parameters that directly impact the learning performance of {\tt TT-HF}. We now prove that with proper choice of these parameters, {\tt TT-HF} obtains sub-linear convergence with rate of $\mathcal{O}(1/t)$. %, which coincides with the convergence rate of SGD in centralized model training.

We now prove that there exist a configuration of the tunable parameters (i.e., $\eta_t, \tau$ and $\epsilon^{(t)}$) under which {\tt TT-HF} achieves sub-linear convergence with rate of $\mathcal{O}(1/t)$.

\begin{theorem} \label{thm:subLin}
 If $\gamma>1/\mu$ and $\alpha\geq\gamma\beta^2/\mu$,
 under Assumptions \ref{beta},~\ref{assump:cons}, and~\ref{assump:SGD_noise}, upon choosing $\eta_t=\frac{\gamma}{t+\alpha}$ and $\epsilon^{(t)}=\eta_t\phi$, $\forall t$, {\tt TT-HF} achieves the following upper bound of convergence: 
    % Using {\tt TT-HF} for ML model training, under Assumption \ref{beta}, if we set the step size as $\eta_t=\frac{\gamma}{t+\alpha}$, $\forall t$, and assuming that $\epsilon(t)=\eta_t\phi$, $\forall t$,
    % we have
    \begin{align} \label{eq:thm2_result-1-T}
        &\mathbb E\left[(F(\hat{\mathbf w}^{(t)})-F(\mathbf w^*))\right]\leq\frac{\nu}{t+\alpha}, ~~\forall t \in \mathcal{T}_k,
        \end{align}
        where {\small $\nu = \left\{\frac{\beta^2\gamma^2Z}{\mu\gamma-1}, \alpha\left[F(\hat{\mathbf w}^{(0)})-F(\mathbf w^*)\right]\right\}$, $Z=
        \frac{1}{2}[\frac{\sigma^2}{\beta}+\frac{2\phi^2}{\beta}]
        +24\left(\varrho^{\mathsf{min}}\right)^{-1}\beta\gamma(\tau-1)\left(1+\frac{\tau-2}{\alpha}\right)
        \left(1+\frac{\tau-1}{\alpha-1}\right)^{4\beta\gamma}\left[\frac{\sigma^2}{\beta}+\frac{\phi^2}{\beta}+\frac{\delta^2}{\beta}\right].$}
        % \begin{equation*}
        % \hspace{-4mm}\resizebox{1\linewidth}{!}{$
        % \begin{aligned}
        % Z&=
        % \frac{1}{2}[\frac{\sigma^2}{\beta}+\frac{2\phi^2}{\beta}]
        % \\&
        % +24\left(\varrho^{\mathsf{min}}\right)^{-1}\beta\gamma(\tau-1)\left(1+\frac{\tau-2}{\alpha}\right)
        % \left(1+\frac{\tau-1}{\alpha-1}\right)^{2\beta\gamma}\left[\frac{\sigma^2}{\beta}+\frac{\phi^2}{\beta}+\frac{\delta^2}{\beta}\right].
        % \end{aligned}
        % $} \hspace{-2mm}
        % \end{equation*}
\end{theorem}
% \vspace{0.05in}

\begin{skproof}
% The complete proof is provided in Appendix~\ref{app:subLin} of our technical report~\cite{techReport}. Here, we summarize the key steps.
We carry out the proof by induction and start with the first global aggregation. The condition in~\eqref{eq:thm2_result-1-T} trivially holds when $t = t_0 = 0$, since $
\nu\geq\alpha[F(\hat{\mathbf w}^{(0)})-F(\mathbf w^*)]$.
    Now, assuming that
    $\mathbb E\left[F(\hat{\mathbf w}^{(t_{k-1})})-F(\mathbf w^*)\right]  
        \leq \frac{\nu}{t_{k-1}+\alpha}$
for some $k\geq 1$, we prove that this implies 
\begin{align} \label{eq:thm2_result-1T}
        &\mathbb E\left[F(\hat{\mathbf w}^{(t)})-F(\mathbf w^*)\right]  
        \leq \frac{\nu}{t+\alpha},\ \forall t\in \mathcal{T}_k,
    \end{align}
    and as a result of which $\mathbb E\left[F(\hat{\mathbf w}^{(t_{k})})-F(\mathbf w^*)\right]  
        \leq \frac{\nu}{t_{k}+\alpha}$.
    % so that the theorem follows by induction over $k$.
    To prove~\eqref{eq:thm2_result-1T}, we use induction over $t\in \{t_{k-1},\dots,t_k-1\}$.
    The condition trivially holds when $t=t_{k-1}$ from the first induction hypothesis.
    Now, we suppose that it holds for some $t\in \{t_{k-1},\dots,t_k-1\}$, and demonstrate that it holds at $t+1$.
    
    From the result of Theorem~\ref{co1},
    % \begin{align} \label{eq:thm1_temp}
    %     &\mathbb E\left[F(\hat{\mathbf w}^{(t+1)})-F(\mathbf w^*)\right]
    %     \leq
    %     (1-\tilde{\mu}\tilde{\eta}_{t})\frac{\nu}{t+\alpha}
    %     +\frac{\tilde{\eta}_{t}\beta}{2}A^{(t)}
    %     +\frac{1}{2}[\tilde{\eta}_{t}^3\tilde{\phi}^2+\tilde{\eta}_{t}^2\tilde{\sigma}^2+\tilde{\eta}_{t+1}^2\tilde{\phi}^2].
    % \end{align} 
using the induction hypothesis, the bound on $A^{(t)}$, $\epsilon^{(t)}=\eta_t\phi$, and the facts that $\eta_{t+1}\leq \eta_t$, $\eta_t\leq \eta_0\leq\frac{\mu}{\beta^2}\leq 1/\beta$ and $\epsilon^{(0)}=\eta_0\phi\leq\phi/\beta$, we get
\begin{equation} \label{eq:thm1_SigT}
    \hspace{-.6mm}\resizebox{.94\linewidth}{!}{$
    \begin{aligned}
        &\mathbb E\left[F(\hat{\mathbf w}^{(t+1)})-F(\mathbf w^*)\right]
        \leq
        \left(1-\mu\eta_{t}\right)\frac{\nu}{t+\alpha} + \frac{\eta_t^2\beta}{2}\left(\sigma^2+2\phi^2\right)
         \\&
        +6\left(\varrho^{\mathsf{min}}\right)^{-1}\eta_t\underbrace{(\Sigma_{t})^2}_{(a)}
        \left(\sigma^2+\phi^2+\delta^2\right)  ,
    \end{aligned}
    $} \hspace{-4mm}
\end{equation}
\vspace{-0.12in}

\noindent where $\Sigma_{t}$ is given in Proposition~\ref{Local_disperseT}.
To bound $(a)$, we first use the fact that 
\vspace{-1mm}
\begin{equation} \label{eq:Sigma_1T}
    \hspace{-6.5mm}\resizebox{.7\linewidth}{!}{$
        \Sigma_{t}
          \leq
          \gamma\beta\underbrace{\Bigg(\prod_{j=t_{k-1}}^{t-1}\Big(1+\frac{2\gamma\beta}{j+\alpha}\Big)\Bigg)}_{(i)}
          \underbrace{\sum_{\ell=t_{k-1}}^{t-1}\frac{1}{\ell+\alpha+\gamma\beta}}_{(ii)}.
          $} \hspace{-7mm}
          \vspace{-2mm}
\end{equation}
Since $\frac{1}{\ell+\alpha+\gamma\beta}$ is decreasing in $\ell$, $(ii)$ can be bounded as
\begin{equation} \label{eq:(ii)}
\resizebox{.8\linewidth}{!}{$
\begin{aligned}
    &\sum_{\ell=t_{k-1}}^{t-1}\frac{1}{\ell+\alpha+\gamma\beta}
    \leq
    % \int_{t_{k-1}-1}^{t-1}\frac{1}{\ell+\alpha+\gamma\beta}\mathrm d\ell
    % \\&
    \ln\left(1+\frac{t-t_{k-1}}{t_{k-1}-1+\alpha+\gamma\beta}\right).
\end{aligned}
$}
\end{equation}
\vspace{-0.05in}

\noindent Rewriting $(i)$ as {\footnotesize$ \hspace{-2mm}
        \displaystyle\prod_{j=t_{k-1}}^{t-1}\hspace{-1.8mm}\big(1+\frac{2\gamma\beta}{j+\alpha}\big)
        =
        e^{\sum_{j=t_{k-1}}^{t-1}\ln\big(1+\frac{2\gamma\beta}{j+\alpha}\big)}
$}, 
% and note that $\ln(1+\frac{\gamma\beta}{j+\alpha})$ is a decreasing function with respect to $j$, 
 we get
 \vspace{-3mm}
{\small\begin{equation} \label{eq:ln(i)-t}
% \hspace{-3mm}\resizebox{.96\linewidth}{!}{$
    \begin{aligned}
        &\sum\limits_{j=t_{k-1}}^{t-1}\ln(1+\frac{2\gamma\beta}{j+\alpha})
        \leq
        % \int_{t_{k-1}-1}^{t-1}\ln(1+\frac{\gamma\beta}{j+\alpha})\mathrm dj
        % \nonumber \\&
        % \leq
        % \gamma\beta\int_{t_{k-1}-1}^{t-1}\frac{1}{j+\alpha}\mathrm dj
        % =
        2\gamma\beta \ln\left(1+\frac{t-t_{k-1}}{t_{k-1}-1+\alpha}\right),
    \end{aligned}
    % $}\hspace{-3mm}
\end{equation}}
 \vspace{-3mm}
 
\noindent which yields {\small$
        \prod_{j=t_{k-1}}^{t-1}\left(1+\frac{2\gamma\beta}{j+\alpha}\right)
        \leq
        \left(1+\frac{t-t_{k-1}}{t_{k-1}-1+\alpha}\right)^{2\gamma\beta}
$.}

    Replacing the bounds for $(i)$ and $(ii)$ back in~\eqref{eq:Sigma_1T},
    % we get:
    % \begin{align} \label{eq:Sigma1_bound}
    %     \Sigma_{+,t}
    %     \leq
    %     \tilde{\gamma}\ln\left(1+\frac{t-t_{k-1}}{t_{k-1}-1+\alpha+\tilde{\gamma}\lambda_+}\right)
    %     \left(1+\frac{t-t_{k-1}}{t_{k-1}-1+\alpha}\right)^{\tilde{\gamma} \lambda_+}.
    % \end{align}
    and using the fact that $\ln(1+x) \leq 2\sqrt{x}$ for $x\geq 0$,
    and performing some algebraic manipulations, we bound $(a)$ in~\eqref{eq:thm1_SigT} as follows:
% \begin{equation*}
% \hspace{-0.1mm}\resizebox{.98\linewidth}{!}{$
% \begin{aligned}
%     (t+\alpha) (\Sigma_{+,t})^2
%     % &\leq
%     % 4\gamma^2\beta^2\frac{(t-t_{k-1})(t+\alpha)}{t_{k-1}+\alpha}
%     % \left(1+\frac{t-t_{k-1}}{t_{k-1}+\alpha-1}\right)^{2\gamma\beta}
%     % \\
%     &\leq
%     4\gamma^2\beta^2(\tau-1)\left(1+\frac{\tau-2}{\alpha}\right)
%     \left(1+\frac{\tau-1}{\alpha-1}\right)^{2\gamma\beta},
% \end{aligned}
% $}
% \end{equation*}
% which implies
\begin{equation} \label{eq:Sigma_2ndP}
\hspace{-.6mm}\resizebox{.85\linewidth}{!}{$
    (\Sigma_{t})^2
    \leq
    4\gamma\beta(\tau-1)\left(1+\frac{\tau-2}{\alpha}\right)
    \left(1+\frac{\tau-1}{\alpha-1}\right)^{4\gamma\beta}\eta_t\beta.
    $}\hspace{-3mm}
\end{equation}
Substituting~\eqref{eq:Sigma_2ndP} into~\eqref{eq:thm1_SigT}, we get
\begin{equation} \label{eq:induce_formT}
\begin{aligned}
    & \mathbb E[F(\hat{\mathbf w}^{(t+1)})-F(\mathbf w^*)]
    % \nonumber \\&
    \leq 
    \left(1-\mu\eta_t\right)\frac{\nu}{t+\alpha}
    +\eta_t^2\beta^2 Z,
\end{aligned}
 \hspace{-1.3mm}
\end{equation}
where $Z$ is given in the statement of the theorem.

% \begin{align}
%     Z_1\triangleq \frac{32\tilde{\gamma}}{\tilde{\mu}}(\tau-1)\left(1+\frac{\tau}{\alpha-1}\right)^{2}
%     \left(1+\frac{\tau-1}{\alpha-1}\right)^{6\tilde{\gamma}},
% \end{align}
% and
% \begin{align}
%     Z_2\triangleq
%     \frac{1}{2}[\tilde{\sigma}^2+2\tilde{\phi}^2]
%     +50\tilde{\gamma}(\tau-1)\left(1+\frac{\tau-2}{\alpha+1}\right)
%     \left(1+\frac{\tau-1}{\alpha-1}\right)^{6\tilde{\gamma}}\left[\tilde{\sigma}^2+\tilde{\phi}^2+\tilde{\delta}^2\right].  
% \end{align}

The induction is completed by showing that the right hand side of~\eqref{eq:induce_formT} is less than or equal to $\frac{\nu}{t+1+\alpha}$, or equivalently:
    \begin{equation}
        \hspace{-.6mm}\resizebox{.94\linewidth}{!}{$
     \begin{aligned}\label{eq:fin2P}
        &\mu\gamma(t+\alpha)\nu
        +Z \gamma^2\beta^2(t+\alpha)
        % \nonumber \\&
        +\nu(t+\alpha-1)
        +\frac{\nu}{t+1+\alpha}
        \leq 0,\forall t.
     \end{aligned}
     $} \hspace{-3.5mm}
    \end{equation}
% needs to be satisfied $\forall t\geq 0$. 
It is sufficient to satisfy \eqref{eq:fin2P} for $t\to\infty $ and $t=0$ since the expression on the left hand side is convex in $t$. Obtaining these limits gives us: $\mu\gamma-1>0$ and $\nu\geq Z\frac{\beta^2\gamma^2}{\mu\gamma-1}$, which completes the induction and thus the proof (See Appendix D).
\end{skproof}

The bound in Theorem~\ref{thm:subLin} reveals the impact of the duration local model training (i.e., $\tau$ encapsulated in $Z$) on the convergence captured in $\nu$; increasing $\tau$ results in a sharp increase of the upper bound. Also, the consensus error $\epsilon^{(t)}$, captured via $\phi$, has a quadratic impact on the upper bound. Fixing all the parameters, increasing $\tau$ requires a smaller value of $\phi$ for a fixed value of $\nu$. This matches well with the intuition behind {\tt TT-HF}, since consensus among the nodes (which decrease $\epsilon^{(t)}$) is incorporated to reduce the global aggregation frequency, and thus saves on costly uplink transmissions.

\begin{remark}\label{rem:D2D_rounds} Using Lemma~\ref{lemma:cons}, the rounds of D2D communications can be tuned to achieve any desired consensus error inside the clusters. In particular, to satisfy the condition in Theorem~\ref{thm:subLin}, it is sufficient to have $
    \Gamma_c^{(t)}=\max\Big\{\log\Big(\frac{\eta_t\phi}{\sqrt{s_c}\Upsilon_c^{(t)}}\Big)/\log\Big(\lambda_c\Big),0\Big\},\forall c$. Sometimes this will yield $\Gamma_c^{(t)}=0$, meaning no current consensus for cluster $c$, and implying consensus formation among nodes is aperiodic.
\end{remark}

\vspace{-1mm}
\section{Numerical Evaluation}
\label{sec:experiments}

% \noindent In this section, we conduct numerical experiments to verify the performance of {\tt TT-HF}. After describing the setup in Sec.~\ref{ssec:setup}, we study model performance/convergence in Sec.~\ref{ssec:conv-eval}. Overall, we will see that {\tt TT-HF} provides substantial improvements in training time, accuracy, and/or resource utilization compared to conventional federated learning~\cite{wang2019adaptive,Li}. 

% \vspace{-2mm}
% \subsection{Experimental Setup}
% \label{ssec:setup}

\subsection{Experimental Setup}
\noindent \textbf{Network architecture.}
We consider a network of $I = 125$ edge devices partitioned into $N = 25$ clusters, each with $s_c = 5$ devices. The links among devices within each cluster are generated using a random geometric graph~\cite{hosseinalipour2020multi}, tuned such that the clusters have an average spectral radius of $\rho = 0.7$. \\
% For the cluster consensus algorithm in~\eqref{eq:consensus}, we follow the distributed average consensus approach outlined above Assumption~\ref{assump:cons}. \vspace{-1mm}
\textbf{Dataset.}
We consider Fashion-MNIST, a dataset commonly used for image classification. It contains $70$K images, where each image is one of 10 labels of fashion products. \\
% For brevity, we present the results for MNIST here, and refer the reader to Appendix~\ref{app:experiments} in our technical report \cite{techReport} for FMNIST; the results are qualitatively similar.\\
\textbf{Local data distributions.}
We partition the images across devices such that each local dataset contains datapoints from only 3 of the 10 labels. The 3 labels are varied across devices. In this way, we consider non-i.i.d. local data distributions.
% $\mathcal{D}_i$ is selected randomly (without replacement) from the full dataset of labels assigned to device $i$. \\
% per node, where each of which corresponds to a particular level of data heterogeneity. (a) represents the case of extreme non-i.i.d, where each device possesses data labeled with only $1$ of the $10$ classes; (b) represents the case of moderate non-i.i.d, where each device samples data from $3$ classes out of $10$; (c) represents the case of i.i.d, where each device has samples uniformly sampled from the entire dataset.
\\
\textbf{ML models.}
We consider two models: regularized (squared) support vector machine (SVM), and a neural network (NN) with one fully connected hidden layer and 7840 neurons.

\subsection{Results and Discussion}
%One of the main premises of {\tt TT-HF} is that cooperative D2D communications within clusters can (i) increase the model accuracy and/or (ii) preserve model performance while decreasing the frequency of global aggregations. 
% Our first set of experiments seek to validate these facts:
% \subsubsection{D2D enhancing ML model performance} 

\textbf{Model improvement from local aggregations:} In Fig.~\ref{fig:SVM/fixed}, we conduct a performance comparison between {\tt TT-HF} and two baselines considering current federated learning (FL) algorithms that do not exploit D2D communications. Both baselines presume full device participation (i.e., all devices conduct uplink transmissions), and thus are 5x more uplink resource-intensive. In one baseline, the global aggregations are performed after each round of training ($\tau = 1$) to replicate centralized training, as an upper bound of performance. In the other baseline, we set $\tau = 20$ based on~\cite{wang2019adaptive}. For {\tt TT-HF}, we set $\tau = 20$ and conduct a fixed number of D2D rounds after every $5$ SGD iterations in all clusters, i.e., $\Gamma^{(t)}_c = \Gamma$ for different $\Gamma$.

Fig.~\ref{fig:SVM/fixed} demonstrates that conducting local D2D communications leads to substantial performance gains in training. 
% Specifically, when the data distributions are moderate non-i.i.d ((b) and (e)) or extreme non-i.i.d. ((c) and (f)), 
Also, it shows that increasing $\Gamma$ leads to better performance gains compared to FL with $\tau = 20$, emphasizing the benefit of local consensus in the presence of non-i.i.d. local data distributions. It further shows a diminishing reward of increasing $\Gamma$ as the performance of {\tt TT-HF} approaches that of FL with $\tau = 1$. 
% Finally, we observe that the gains obtained through D2D communications are only present when the data distributions across the nodes are non-i.i.d., as compared to the i.i.d. scenario ((a) and (d)), which emphasizes the purpose of {\tt TT-HF} for handling statistical heterogeneity.

% \subsubsection{Local consensus reducing global aggregation frequency}

\textbf{Reduction in global aggregation frequency:} In Fig.~\ref{fig:SVM/varied}, we conduct a performance comparison between {\tt TT-HF} and the baselines for increased local model training intervals $\tau$; recall that larger $\tau$ reduces the frequency of uplink communications. We conduct consensus after every $5$ SGD iterations, and increase $\Gamma$ as $\tau$ increases. We see that {\tt TT-HF} outperforms the FL baseline with $\tau = 20$ while utilizing a lower frequency of global aggregations: increasing $\tau$ can be counteracted with increased D2D rounds $\Gamma$ among the nodes. %We also see that the jumps in global model performance are larger for {\tt TT-HF} than this baseline.

\textbf{Improvement in energy and delay:} 
% and Fig.~\ref{fig:resource/delay} 
Finally, we consider performance in terms of energy consumption and training delay incurred. In Fig.~\ref{fig:resource/power}, we compare {\tt TT-HF} against (i) FL with full device participation and $\tau = 1$, and (ii) FL with only one device randomly selected from each cluster and $\tau = 20$, for NN.\footnote{Similar results for SVM are observed, omitted due to space limitations.} For {\tt TT-HF}, we increase the local model training interval to $\tau = 40$ and conduct aperiodic D2D consensus rounds according to Remark \ref{rem:D2D_rounds}. We demonstrate the result under various ratios of energy consumption $E_{\textrm{D2D}} / E_{\textrm{Glob}}$ and delays $\Delta_{\textrm{D2D}} / \Delta_{\textrm{Glob}}$ between D2D communications and global aggregations. For uplink transmission, we assumed that each device transmits with a power of $24\textrm{dbm}$ and a delay of $0.25\textrm{s}$~\cite{hmila2019energy}.

In Fig.~\ref{fig:resource/power}(a), we see that {\tt TT-HF} lowers the overall energy consumption for smaller values of ${E_{\textrm{D2D}}}/{E_{\textrm{Glob}}}$. After the ratio reaches a certain threshold, {\tt TT-HF} no longer saves energy, as we would expect. Similarly, in Fig.~\ref{fig:resource/power}(b), the performance gain of {\tt TT-HF} narrows as $\Delta_{\textrm{D2D}} / \Delta_{\textrm{Glob}}$ increases. Ratios of $0.1$ for either of these metrics is significantly larger than what is being observed in 5G~\cite{hmila2019energy}, indicating that {\tt TT-HF} would be effective in practical systems.

\begin{figure}[t]
% \centering
% \hspace*{0.2in}
% \begin{subfigure}{0.2\textwidth}
% \includegraphics[height=29mm,width=32mm]{fmnist_125/acc_125_1.pdf}
% \subcaption{SVM with same $\tau$}
% \label{fig:SVM/fixed}
% \end{subfigure} \hspace{0.02\textwidth}
\begin{subfigure}{0.2\textwidth}
\includegraphics[height=30mm,width=85mm]{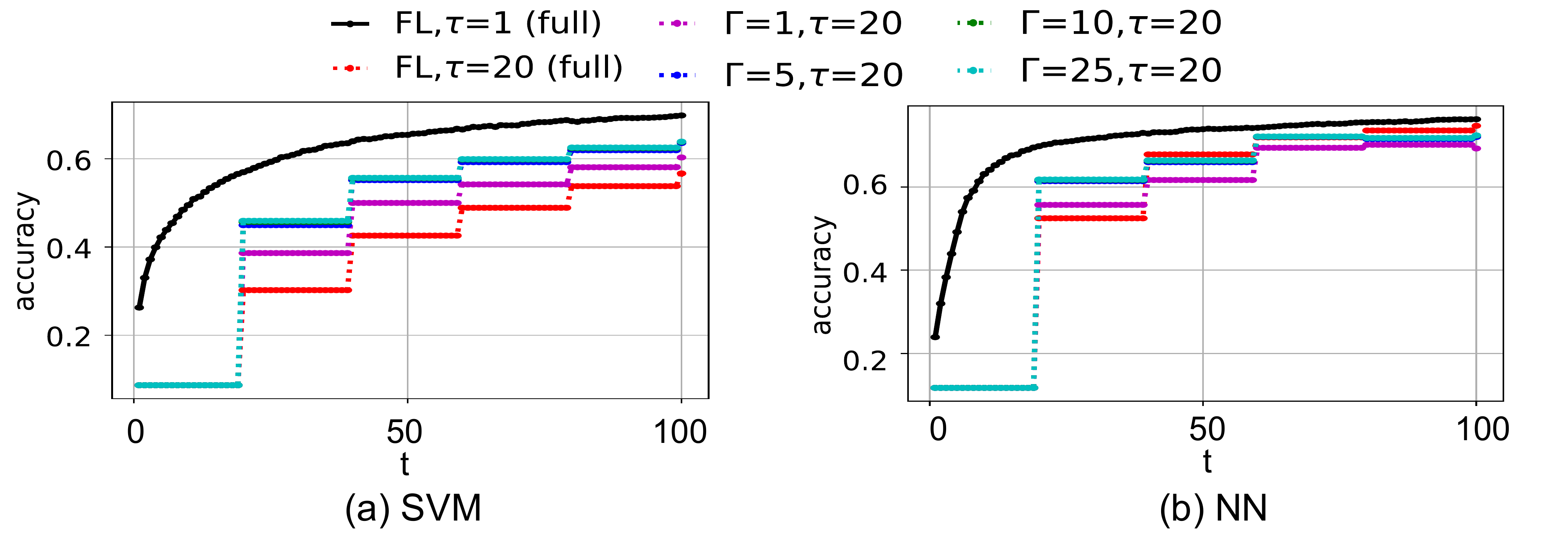}
% \caption{SVM with different $\tau$}
% \label{fig:SVM/fixed}
\end{subfigure}
\vspace{-1.mm}
\caption{Performance comparison between {\tt TT-HF} and baseline methods when varying the number of D2D consensus rounds ($\Gamma$).}
\label{fig:SVM/fixed}
\vspace{-3mm}
\end{figure}

\begin{figure}[t]
%  \centering
% \hspace*{0.2in}
\begin{subfigure}{0.2\textwidth}
\includegraphics[height=30mm,width=85mm]{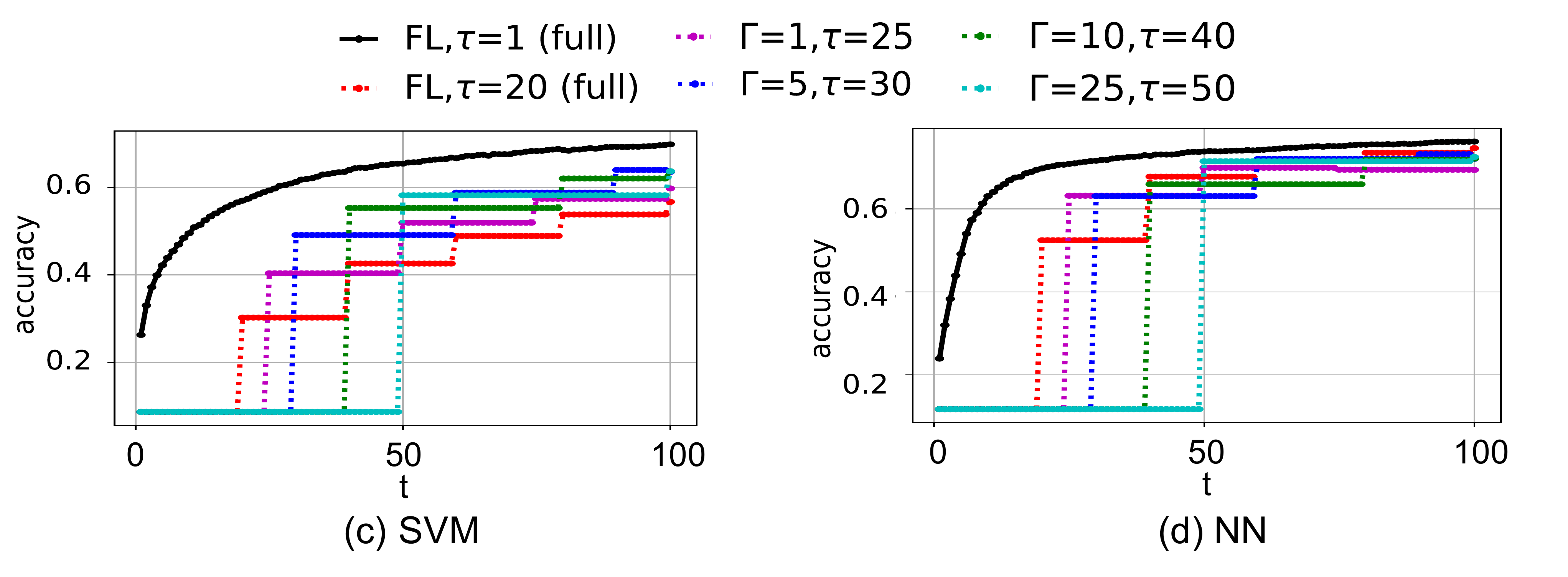}
% \caption{SVM with different $\tau$}
% \label{fig:SVM/varied}
\end{subfigure}
\vspace{-1.mm}
\caption{Performance comparison between {\tt TT-HF} and baseline methods varying $\tau$ and the number of D2D consensus rounds ($\Gamma$).}
\label{fig:SVM/varied}
\vspace{-4mm}
\end{figure} 

\begin{figure}[t]
% \hspace*{0.2in}
\begin{subfigure}{0.2\linewidth}
% \hspace*{-5mm}
\includegraphics[height=30mm,width=85mm]{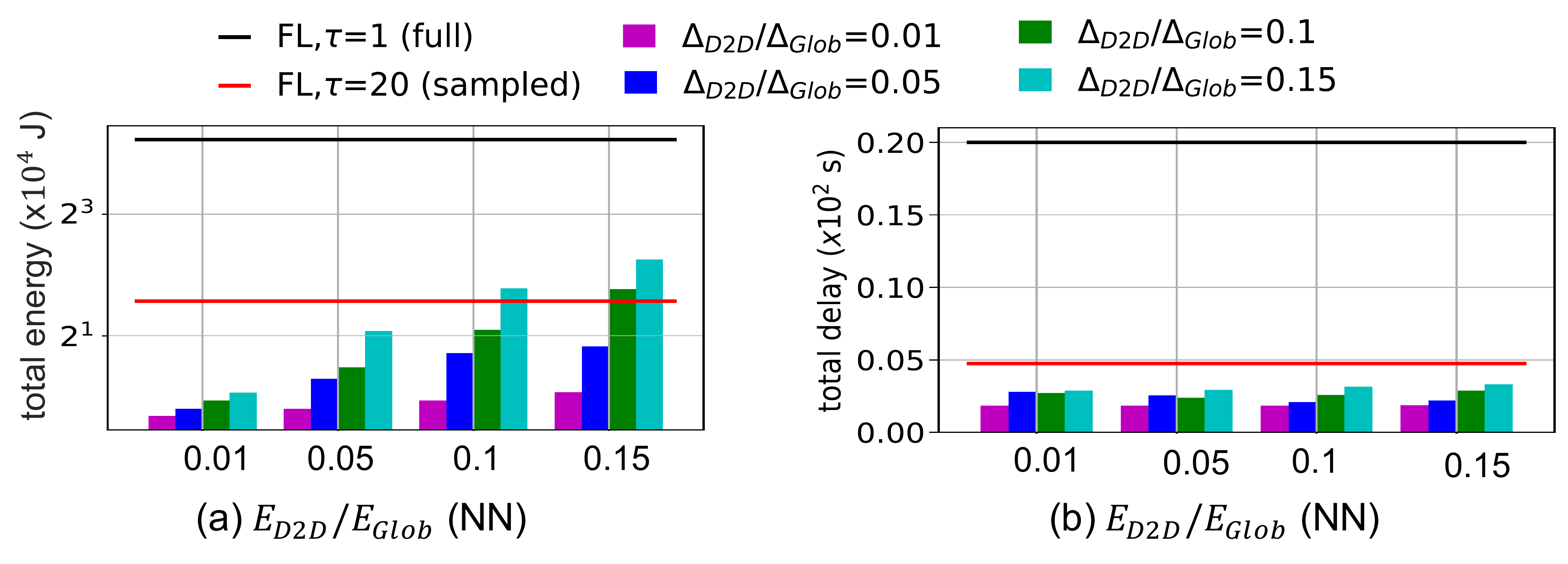}
% \caption{SVM-total power}
% \label{fig:SVM/fixed}
\end{subfigure}
\vspace{-1.mm}
\caption{Comparing total energy and delay achieved by {\tt TT-HF} versus baselines upon reaching $60\%$ of peak accuracy for different configurations of delay and energy consumption in the case of NN.}
\label{fig:resource/power}
\vspace{-5mm}
\end{figure}

\section{Conclusion} 
\noindent We developed {\tt TT-HF}, a methodology that augments the star topology of conventional federated learning with cooperative consensus among devices in D2D-enabled edge networks. We investigated the convergence behavior of {\tt TT-HF}, revealing the impact of the consensus error, gradient diversity, and global aggregation period on convergence. We then identified a set of conditions under which {\tt TT-HF} converges sublinearly with rate of $\mathcal{O}(1/t)$, coinciding with centralized SGD. Through numerical experiments, we demonstrated the performance gains that can be achieved via {\tt TT-HF} in terms of model accuracy, training time, and network resource utilization.

% \begin{figure}[h]
% % \centering
% \begin{subfigure}{0.2\linewidth}
% \hspace*{-3mm}
% \includegraphics[height=28mm,width=23mm]{fmnist_125/resource_a.pdf}
% % \hspace*{0mm}\subcaption{SVM-power}
% % \label{fig:SVM/fixed}
% \end{subfigure} \hspace{1mm}
% \begin{subfigure}{0.2\linewidth}
% % \vspace{0.4in}
% \includegraphics[height=28mm,width=23mm]{fmnist_125/resource_b.pdf}
% % \subcaption{SVM-power}
% % \label{fig:SVM/varied}
% \end{subfigure} \hspace{3mm}
% \begin{subfigure}{0.2\linewidth}
% % \hspace*{-5mm}
% \includegraphics[height=28mm,width=23mm]{fmnist_nn/resource_ap.pdf}
% % \subcaption{NN-power}
% % \label{fig:SVM/fixed}
% \end{subfigure} \hspace{3mm}
% \begin{subfigure}{0.2\linewidth}
% \vspace{0.2in}
% \includegraphics[height=25mm,width=23mm]{fmnist_nn/resource_bp.pdf}
% % \subcaption{SNN-delay}
% % \label{fig:SVM/varied}
% \end{subfigure}
% \caption{Comparing total power (a) (a'), and delay metrics (b) (b') achieved by {\tt TT-HF} versus baselines upon reaching $60\%$ of peak accuracy, for different configurations of delay and energy consumption.}
% \label{fig:resource_bar_0}
% \vspace*{-3mm}
% \end{figure}

% There are several avenues for future work. To further enhance the flexibility of {\tt TT-HF}, one may consider (i) heterogeneity in computation capabilities across edge devices, (ii) different communication delays from the clusters to the server, and (iii) wireless interference caused by D2D communications.

% \pagebreak
\vspace{-2mm}
\bibliographystyle{IEEEtran}
\bibliography{ref}

\begin{thebibliography}{10}
\providecommand{\url}[1]{#1}
\csname url@samestyle\endcsname
\providecommand{\newblock}{\relax}
\providecommand{\bibinfo}[2]{#2}
\providecommand{\BIBentrySTDinterwordspacing}{\spaceskip=0pt\relax}
\providecommand{\BIBentryALTinterwordstretchfactor}{4}
\providecommand{\BIBentryALTinterwordspacing}{\spaceskip=\fontdimen2\font plus
\BIBentryALTinterwordstretchfactor\fontdimen3\font minus
  \fontdimen4\font\relax}
\providecommand{\BIBforeignlanguage}[2]{{%
\expandafter\ifx\csname l@#1\endcsname\relax
\typeout{** WARNING: IEEEtran.bst: No hyphenation pattern has been}%
\typeout{** loaded for the language `#1'. Using the pattern for}%
\typeout{** the default language instead.}%
\else
\language=\csname l@#1\endcsname
\fi
#2}}
\providecommand{\BIBdecl}{\relax}
\BIBdecl

\bibitem{lin2021timescale}
F.~P.-C. Lin, S.~Hosseinalipour, S.~S. Azam, C.~G. Brinton, and N.~Michelusi,
  ``Semi-decentralized federated learning with cooperative {D2D} local model
  aggregations,'' \emph{IEEE J. Sel. Areas Commun.}, 2021.

\bibitem{frank2020delay}
F.~P.-C. Lin, C.~G. Brinton, and N.~Michelusi, ``Federated learning with
  communication delay in edge networks,'' in \emph{Proc. IEEE Int. Glob.
  Commun. Conf. (GLOBECOM)}, 2020, pp. 1--6.

\bibitem{azam2020towards}
S.~S. Azam, T.~Kim, S.~Hosseinalipour, C.~Brinton, C.~Joe-Wong, and S.~Bagchi,
  ``Towards generalized and distributed privacy-preserving representation
  learning,'' \emph{arXiv preprint arXiv:2010.01792}, 2020.

\bibitem{hosseinalipour2020federated}
S.~{Hosseinalipour}, C.~G. {Brinton}, V.~{Aggarwal}, H.~{Dai}, and M.~{Chiang},
  ``From federated to fog learning: Distributed machine learning over
  heterogeneous wireless networks,'' \emph{IEEE Commun. Mag.}, vol.~58, no.~12,
  pp. 41--47, 2020.

\bibitem{rahman2020survey}
S.~A. Rahman, H.~Tout, H.~Ould-Slimane, A.~Mourad, C.~Talhi, and M.~Guizani,
  ``A survey on federated learning: The journey from centralized to distributed
  on-site learning and beyond,'' \emph{IEEE Internet Things J.}, 2020.

\bibitem{wang2019adaptive}
S.~Wang, T.~Tuor, T.~Salonidis, K.~K. Leung, C.~Makaya, T.~He, and K.~Chan,
  ``Adaptive federated learning in resource constrained edge computing
  systems,'' \emph{IEEE J. Select. Areas Commun.}, vol.~37, no.~6, pp.
  1205--1221, 2019.

\bibitem{tu2020network}
Y.~Tu, Y.~Ruan, S.~Wagle, C.~Brinton, and C.~Joe-Wong, ``Network-aware
  optimization of distributed learning for fog computing,'' in \emph{Proc. IEEE
  Int. Conf. Comput. Comun. (INFOCOM)}, 2020, pp. 2509--2518.

\bibitem{liu2020client}
L.~Liu, J.~Zhang, S.~Song, and K.~B. Letaief, ``Client-edge-cloud hierarchical
  federated learning,'' in \emph{Proc. IEEE Int. Conf. Commun. (ICC)}, 2020,
  pp. 1--6.

\bibitem{9149323}
N.~{Yoshida}, T.~{Nishio}, M.~{Morikura}, K.~{Yamamoto}, and R.~{Yonetani},
  ``Hybrid-fl for wireless networks: Cooperative learning mechanism using
  non-iid data,'' in \emph{Proc. IEEE Int. Conf. Commun. (ICC)}, 2020, pp.
  1--7.

\bibitem{zhao2018federated}
Y.~Zhao, M.~Li, L.~Lai, N.~Suda, D.~Civin, and V.~Chandra, ``Federated learning
  with non-iid data,'' \emph{arXiv preprint arXiv:1806.00582}, 2018.

\bibitem{8950073}
S.~{Savazzi}, M.~{Nicoli}, and V.~{Rampa}, ``Federated learning with
  cooperating devices: A consensus approach for massive iot networks,''
  \emph{IEEE Internet Things J.}, vol.~7, no.~5, pp. 4641--4654, 2020.

\bibitem{9154332}
H.~{Xing}, O.~{Simeone}, and S.~{Bi}, ``Decentralized federated learning via
  {SGD} over wireless {D2D} networks,'' in \emph{IEEE Int. Workshop Signal
  Process. Adv. Wireless Commun. (SPAWC)}, 2020, pp. 1--5.

\bibitem{hosseinalipour2020multi}
S.~Hosseinalipour, S.~S. Azam, C.~G. Brinton, N.~Michelusi, V.~Aggarwal, D.~J.
  Love, and H.~Dai, ``Multi-stage hybrid federated learning over large-scale
  wireless fog networks,'' \emph{arXiv preprint arXiv:2007.09511}, 2020.

\bibitem{chen2019joint}
M.~Chen, Z.~Yang, W.~Saad, C.~Yin, H.~V. Poor, and S.~Cui, ``A joint learning
  and communications framework for federated learning over wireless networks,''
  \emph{IEEE Trans. Wireless Commun.}, 2020.

\bibitem{xiao2004fast}
L.~Xiao and S.~Boyd, ``Fast linear iterations for distributed averaging,''
  \emph{Syst. \& Control Lett.}, vol.~53, no.~1, pp. 65--78, 2004.

\bibitem{Bubeck}
S.~Bubeck \emph{et~al.}, ``Convex optimization: Algorithms and complexity,''
  \emph{Found. Trends{\textregistered} Machine Learn.}, vol.~8, no. 3-4, pp.
  231--357, 2015.

\bibitem{hmila2019energy}
M.~Hmila, M.~Fern{\'a}ndez-Veiga, M.~Rodr{\'\i}guez-P{\'e}rez, and
  S.~Herrer{\'\i}a-Alonso, ``Energy efficient power and channel allocation in
  underlay device to multi device communications,'' \emph{IEEE Trans. Commun.},
  vol.~67, no.~8, pp. 5817--5832, 2019.

\end{thebibliography}

\pagebreak

\begingroup
\let\clearpage\relax 
% \onecolumn %%% For
\onecolumn

\appendices 
\setcounter{lemma}{0}
\setcounter{proposition}{0}
\setcounter{theorem}{0}
\setcounter{definition}{0}
\setcounter{assumption}{0}
\section{Preliminaries and Notations used in the Proofs}\label{app:notations}
In the following Appendices, in order to increase the tractability of the the expressions inside the proofs, we introduce the  the following scaled parameters: (i) strong convexity denoted by $\tilde{\mu}$, normalized gradient diversity by $\tilde{\delta}$, step size by $\tilde{\eta_t}$, SGD variance $\tilde{\sigma}$, and consensus error inside the clusters $\tilde{\epsilon}_c^{(t)}$ and across the network $\tilde{\epsilon}^{(t)}$ inside the cluster as follows:
\begin{itemize}[leftmargin=5mm]
\item  \textbf{Strong convexity}:
 $F$ is $\mu$-strongly convex, i.e.,
\begin{align} 
    F(\mathbf w_1) \geq  F(\mathbf w_2)&+\nabla F(\mathbf w_2)^\top(\mathbf w_1-\mathbf w_2)+\frac{\tilde{\mu}\beta}{2}\Big\Vert\mathbf w_1-\mathbf w_2\Big\Vert^2,~\forall { \mathbf w_1,\mathbf w_2},
\end{align}
where as compared to~Assumption~\ref{Assump:SmoothStrong}, we considered $\tilde{\mu}=\mu/\beta\in(0,1)$.
\item \textbf{Gradient diversity}: The gradient diversity across the device clusters $c$ is measured via two non-negative constants $\delta,\zeta$ that satisfy 
\begin{align} 
    \Vert\nabla\hat F_c(\mathbf w)-\nabla F(\mathbf w)\Vert
    \leq \sqrt{\beta}\tilde{\delta}+ 2\omega\beta \Vert\mathbf w-\mathbf w^*\Vert,~\forall c, \mathbf w,
\end{align}
where as compared to Assumption~\ref{gradDiv}, we presumed $\tilde{\delta}=\delta/\sqrt{\beta}$ and $\omega=\zeta/(2\beta)\in [0,1]$.

\item \textbf{Step size}: The local updates to compute \textit{intermediate updated local model} at the devices is expressed as follows:
\begin{align} 
    {\widetilde{\mathbf{w}}}_i^{(t)} = 
           \mathbf w_i^{(t-1)}-\frac{\tilde{\eta}_{t-1}}{\beta} \widehat{\mathbf g}_{i}^{(t-1)},~t\in\mathcal T_k,
\end{align}
where we used the scaled in the step size, i.e., ${\tilde{\eta}_{t-1}}=\eta_{t-1}{\beta}$. Also, when we consider decreasing step size, we consider scaled parameter $\tilde{\gamma}$ in the step size as follows: $\frac{\gamma}{t+\alpha}=\frac{\tilde{\gamma}/\beta}{t+\alpha}$ indicating that $\tilde{\gamma}=\gamma\beta$.

\item \textbf{Variance of the noise of the estimated gradient through SGD}:
The variance on the SGD noise is bounded as:
\begin{align}
    \mathbb{E}[\Vert{\mathbf n}_{j}^{(t)}\Vert^2]\leq \beta\tilde{\sigma}^2, \forall j,t,
\end{align}
where we consider scaled SGD noise as: $\tilde{\sigma}^2=\sigma^2/\beta$.

\item \textbf{Average of the consensus error inside cluster $c$ and across the network}: $\epsilon_c^{(t)}$ is an upper bound on the average of the consensus error inside cluster $c$ for time $t$, i.e.,
    \begin{align}
        \frac{1}{s_c}\sum\limits_{i\in \mathcal S_c}\Vert \mathbf{e}_i^{(t)}\Vert^2 \leq (\tilde{\epsilon}_c^{(t)})^2/\beta,
    \end{align}
    where we use the scaled consensus error $(\tilde{\epsilon}_c^{(t)})^2=\beta(\epsilon_c^{(t)})^2$.
    % When the consensus is assumed to be decreasing over time, we use the scaled coefficient $\tilde{\phi}_c$, where ... ....
    % $(\epsilon_c^{(t)})^2=(\tilde{\epsilon}_c^{(t)})^2/\beta=\eta_t^2\phi_c^2=\tilde{\eta}_t^2\tilde{\phi_c}^2/\beta$, which indicates that $\phi_c^2=\tilde{\phi_c}^2/\beta$. 
    Also, in the proofs we use the notation $\epsilon$ to denote the average consensus error across the network defined as $(\epsilon^{(t)})^2=\sum_{c=1}^N\varrho_c(\epsilon_c^{(t)})^2$. When the consensus is assumed to be decreasing over time we use the scaled coefficient $\tilde{\phi}^2=\phi^2/\beta$, resulting in  $(\epsilon^{(t)})^2=\eta_t^2\tilde{\phi}^2\beta $.
    % \begin{align}
    %  \sum_{c=1}^N\varrho_c\phi_c^2=\phi^2,
    % \end{align}
    % where $\phi^2=\tilde{\phi}^2/\beta$.
\end{itemize}

\section{Proof of Proposition \ref{Local_disperse}} \label{app:Local_disperse}
\begin{proposition} \label{Local_disperse}
    Under Assumptions \ref{beta} and~\ref{assump:SGD_noise}, if $\eta_t=\frac{\gamma}{t+\alpha}$, $\epsilon^{(t)}$ is non-increasing with respect to $t\in \mathcal T_k$, i.e., $\epsilon^{(t+1)}/\epsilon^{(t)} \leq 1$ and $\alpha\geq\beta^2\gamma/\mu$, using {\tt TT-HF} for ML model training, the following upper bound on the expected model dispersion across the clusters holds:
        \begin{align}
            &A^{(t)}\leq
             12\left(\varrho^{\mathsf{min}}\right)^{-1}[\Sigma_{t}]^2
            \left[\frac{\sigma^2}{\beta^2}+\frac{\delta^2}{\beta^2}+(\epsilon^{(0)})^2\right], ~~t\in\mathcal{T}_k,
        \end{align}
        where $\varrho^{\mathsf{min}}=\min_c\varrho_c$ and
        \begin{align}
            &[\Sigma_{t}]^2=\left[\sum\limits_{\ell=t_{k-1}}^{t-1}\beta\eta_\ell\left(\prod_{j=\ell+1}^{t-1}(1+2\eta_j\beta)\right)\right]^2.
        \end{align}
\end{proposition}
% \addFL{
% \textbf{Condition for} $\Sigma_{-,t}\geq0$: from the following expression  
% % \begin{align}
% %     \Sigma_{\{+,-,0\},t}=\sum\limits_{\ell=t_{k-1}}^{t-1}\left(\prod_{j=t_{k-1}}^{\ell-1}(1+\eta_j\beta\lambda_{\{+,-,0\}})\right)\beta\eta_\ell\left(\prod_{j=\ell+1}^{t-1}(1+\eta_j\beta)\right),
% % \end{align}
% \begin{align}
%     \Sigma_{-,t}=\sum\limits_{\ell=t_{k-1}}^{t-1}\left(\prod_{j=t_{k-1}}^{\ell-1}(1+\eta_j\beta\lambda_{-})\right)\beta\eta_\ell\left(\prod_{j=\ell+1}^{t-1}(1+\eta_j\beta)\right),
% \end{align}
% where $\lambda_{\pm}=1-\frac{\vartheta}{4}\pm\sqrt{(1+\frac{\vartheta}{4})^2+2\omega}$ and $\lambda_{-}=-\frac{\vartheta+2\omega}{\lambda_+}<0$.
% We can observe that the condition for $\Sigma_{-,t}\geq0$ is $$1+\eta_t\beta\lambda_{-}\geq0,~\forall t \in \mathcal T_k,$$ 
% which can be satisfied when
% $$
% 1+\frac{\gamma\beta\lambda_-}{\alpha}\geq0,
% $$
% or equivalently 
% $$
% \gamma\leq-\frac{\alpha}{\beta\lambda_-}.
% $$
% Since $2\leq\lambda_+\leq3$ and $0\leq\omega\leq1$, we have
% $$
% -\frac{\vartheta+2}{2}\leq-\frac{\vartheta+2\omega}{2}\leq\lambda_-\leq-\frac{\vartheta+2\omega}{3}\leq-\frac{\vartheta}{3}.
% $$
% Then the condition on gamma to satisfy all possible values of $\lambda_-$ to make $\Sigma_{-,t}\geq0$ is
% $$
% \gamma\leq\frac{2\alpha}{\beta(\vartheta+2)}.
% $$
% }
\begin{proof}
% \addFL{

We break down the proof into 3 parts: in Part I we find the relationship between $\Vert\bar{\mathbf w}^{(t)}-\mathbf w^*\Vert$ and $\sum\limits_{c=1}^N\varrho_{c}\Vert\bar{\mathbf w}_c^{(t)}-\bar{\mathbf w}^{(t)}\Vert$, which turns out to form a coupled dynamic system, which is solved in Part II. Finally, Part III draws the connection between $A^{(t)}$ and the solution of the coupled dynamic system and obtains the upper bound on $A^{(t)}$.

\textbf{(Part I) Finding the upper bound of $\sum\limits_{c=1}^N\varrho_{c}\Vert\bar{\mathbf w}_c^{(t)}-\bar{\mathbf w}^{(t)}\Vert$}:
Using the definition of $\bar{\mathbf w}^{(t+1)}$ given in Definition~\ref{modDisp}, and the notations introduced in Appendix~\ref{app:notations}, we have:  
\begin{equation}
    \bar{\mathbf w}^{(t+1)}= \bar{\mathbf w}^{(t)}-\frac{\tilde{\eta_t}}{\beta}\sum\limits_{c=1}^N\varrho_{c}\frac{1}{s_{c}}\sum\limits_{j\in\mathcal S_{c}} \widehat{\mathbf g}_{j,t},~t\in\mathcal{T}_k.
\end{equation}

    We use the fact that $\bar{\mathbf w}_c^{(t+1)}$ can be written as follows:
    \begin{align} \label{eq:w_c}
        &\bar{\mathbf w}_c^{(t+1)}=\bar{\mathbf w}_c^{(t)}
        -\frac{\tilde{\eta}_t}{\beta}\frac{1}{s_{c}}\sum\limits_{j\in\mathcal S_{c}}\nabla F_j(\mathbf w_j^{(t)})
        -\frac{\tilde{\eta}_t}{\beta}\frac{1}{s_{c}}\sum\limits_{j\in\mathcal S_{c}}\mathbf n_j^{(t)}.
    \end{align}
    Similarly, $\bar{\mathbf w}^{(t+1)}$ can be written as:
    \begin{align} \label{eq:w-}
        &\bar{\mathbf w}^{(t+1)}=
        \bar{\mathbf w}^{(t)}
        -\frac{\tilde{\eta}_t}{\beta}\sum\limits_{d=1}^N\varrho_{d}\frac{1}{s_{d}}\sum\limits_{j\in\mathcal S_{d}}\nabla F_j(\mathbf w_j^{(t)})
        -\frac{\tilde{\eta}_t}{\beta}\sum\limits_{d=1}^N\varrho_{d}\frac{1}{s_{d}}\sum\limits_{j\in\mathcal S_{d}}\mathbf n_j^{(t)}.
    \end{align}
    Combining~\eqref{eq:w_c} and~\eqref{eq:w-} and performing some algebraic manipulations yields:
    \begin{align} 
        &\bar{\mathbf w}_c^{(t+1)}-\bar{\mathbf w}^{(t+1)}=\bar{\mathbf w}_c^{(t)}-\bar{\mathbf w}^{(t)}
        -\frac{\tilde{\eta}_t}{\beta}\frac{1}{s_{c}}\sum\limits_{j\in\mathcal S_{c}}\mathbf n_{j}^{(t)}
        +\frac{\tilde{\eta}_t}{\beta}\sum\limits_{d=1}^N\varrho_{d}\frac{1}{s_{d}}\sum\limits_{j\in\mathcal S_{d}}\mathbf n_{j}^{(t)}
        \nonumber \\&
        -\frac{\tilde{\eta}_t}{\beta}\frac{1}{s_{c}}\sum\limits_{j\in\mathcal S_{c}}\Big[\nabla F_j(\bar{\mathbf w}_j ^{(t)})-\nabla F_j(\bar{\mathbf w}_c ^{(t)})\Big]
        +\frac{\tilde{\eta}_t}{\beta}\sum\limits_{d=1}^N\varrho_{d}\frac{1}{s_{d}}\sum\limits_{j\in\mathcal S_{d}}\Big[\nabla F_j(\bar{\mathbf w}_j ^{(t)})-\nabla F_j(\bar{\mathbf w}_d^{(t)})\Big]
        \nonumber \\&
        -\frac{\tilde{\eta}_t}{\beta}\Big[\nabla\hat F_c(\bar{\mathbf w}_c^{(t)})-\nabla\hat F_c(\bar{\mathbf w}^{(t)})\Big]
        +\frac{\tilde{\eta}_t}{\beta}\sum\limits_{d=1}^N\varrho_{d}\Big[\nabla\hat F_d(\bar{\mathbf w}_d^{(t)})-\nabla\hat F_d(\bar{\mathbf w}^{(t)})\Big]
       \nonumber \\&
        -\frac{\tilde{\eta}_t}{\beta}\Big[\nabla\hat F_c(\bar{\mathbf w}^{(t)})-\nabla F(\bar{\mathbf w}^{(t)})\Big].
    \end{align}   
    Taking the norm-2 of the both hand sides of the above equality and applying the triangle inequality gives us
    \begin{align} \label{eq:tri_wc}
        &\Vert\bar{\mathbf w}_c^{(t+1)}-\bar{\mathbf w}^{(t+1)}\Vert\leq
        \Vert\bar{\mathbf w}_c^{(t)}-\bar{\mathbf w}^{(t)}\Vert
        +\frac{\tilde{\eta}_t}{\beta}\Vert\frac{1}{s_{c}}\sum\limits_{j\in\mathcal S_{c}}\mathbf n_{j}^{(t)}\Vert
        +\frac{\tilde{\eta}_t}{\beta}\Vert\sum\limits_{d=1}^N\varrho_{d}\frac{1}{s_{d}}\sum\limits_{j\in\mathcal S_{d}}\mathbf n_{j}^{(t)}\Vert
        \nonumber \\&
        +\frac{\tilde{\eta}_t}{\beta}\frac{1}{s_{c}}\sum\limits_{j\in\mathcal S_{c}}\Vert\nabla F_j(\bar{\mathbf w}_j ^{(t)})-\nabla F_j(\bar{\mathbf w}_c ^{(t)})\Vert
        +\frac{\tilde{\eta}_t}{\beta}\sum\limits_{d=1}^N\varrho_{d}\frac{1}{s_{d}}\sum\limits_{j\in\mathcal S_{d}}\Vert\nabla F_j(\bar{\mathbf w}_j ^{(t)})-\nabla F_j(\bar{\mathbf w}_d^{(t)})\Vert
        \nonumber \\&
        +\frac{\tilde{\eta}_t}{\beta}\Vert\nabla\hat F_c(\bar{\mathbf w}_c^{(t)})-\nabla\hat F_c(\bar{\mathbf w}^{(t)})\Vert
        +\frac{\tilde{\eta}_t}{\beta}\sum\limits_{d=1}^N\varrho_{d}\Vert\nabla\hat F_d(\bar{\mathbf w}_d^{(t)})-\nabla\hat F_d(\bar{\mathbf w}^{(t)})\Vert
        \nonumber \\&
        +\frac{\tilde{\eta}_t}{\beta}\Vert\nabla\hat F_c(\bar{\mathbf w}^{(t)})-\nabla F(\bar{\mathbf w}^{(t)})\Vert.
    \end{align}   
    Considering the terms on the right hand side of~\eqref{eq:tri_wc}, using $\beta$-smoothness, we have 
    \begin{align} \label{eq:beta_use}
        \Vert\nabla F_j(\mathbf w_j^{(t)})-\nabla F_j(\bar{\mathbf w}_c^{(t)})\Vert
        \leq
        \beta\frac{1}{s_{c}}\sum\limits_{j\in\mathcal S_{c}}\Vert\mathbf w_j^{(t)}-\bar{\mathbf w}_c^{(t)}\Vert.
    \end{align}
    Moreover, using Definition \ref{paraDiv}, we get
    \begin{align} \label{eq:e_c}
        \frac{1}{s_{c}}\sum\limits_{j\in\mathcal S_{c}}\Vert\mathbf w_j^{(t)}-\bar{\mathbf w}_c^{(t)}\Vert
        &=
        \frac{1}{s_{c}}\sum\limits_{j\in\mathcal S_{c}}\Vert\mathbf e_j^{(t)}\Vert
        % \nonumber \\&
        \leq
        \sqrt{\frac{1}{s_{c}}\sum\limits_{j\in\mathcal S_{c}}\Vert\mathbf e_j^{(t)}\Vert^2}
        \leq\tilde{\epsilon}_c^{(t)}/\sqrt{\beta}.
    \end{align}
% Combining \eqref{eq:beta_use} and \eqref{eq:e_c} gives us:
%     \begin{align} \label{eq:epsilon_use}
%         \frac{1}{s_{c}}\sum\limits_{j\in\mathcal S_{c}}\Vert\nabla F_j(\mathbf w_j^{(t)})-\nabla F_j(\bar{\mathbf w}_c^{(t)})\Vert
%         \leq
%         \sqrt{\beta}\tilde{\epsilon}_c^{(t)}.
%     \end{align}
   Replacing \eqref{eq:beta_use} into \eqref{eq:tri_wc}, we further bound the right hand side as
\begin{align} \label{eq:tri_wc2}
        &\Vert\bar{\mathbf w}_c^{(t+1)}-\bar{\mathbf w}^{(t+1)}\Vert\leq
        (1+\tilde{\eta}_t)\Vert\bar{\mathbf w}_c^{(t)}-\bar{\mathbf w}^{(t)}\Vert
                +\tilde{\eta}_t\sum\limits_{d=1}^N\varrho_{d}\Vert\bar{\mathbf w}_d^{(t)}-\bar{\mathbf w}^{(t)}\Vert
        \nonumber \\&
        +\frac{\tilde{\eta}_t}{\beta}\Vert\frac{1}{s_{c}}\sum\limits_{j\in\mathcal S_{c}}\mathbf n_{j}^{(t)}\Vert
        +\frac{\tilde{\eta}_t}{\beta}\Vert\sum\limits_{d=1}^N\varrho_{d}\frac{1}{s_{d}}\sum\limits_{j\in\mathcal S_{d}}\mathbf n_{j}^{(t)}\Vert
        \nonumber \\&
        +\tilde{\eta}_t\frac{1}{s_{c}}\sum\limits_{j\in\mathcal S_{c}}
        \Vert\bar{\mathbf w}_j ^{(t)}-\bar{\mathbf w}_c ^{(t)}\Vert
        +\tilde{\eta}_t\sum\limits_{d=1}^N\varrho_{d}\frac{1}{s_{d}}\sum\limits_{j\in\mathcal S_{d}}
        \Vert\bar{\mathbf w}_j ^{(t)}-\bar{\mathbf w}_d^{(t)}\Vert
        +\frac{\tilde{\eta}_t}{\sqrt{\beta}}\tilde{\delta}.
    \end{align}   
    Using~\eqref{eq:e_c} we have $\frac{1}{s_{d}}\sum\limits_{j\in\mathcal S_{d}}
        \Vert\bar{\mathbf w}_j ^{(t)}-\bar{\mathbf w}_d^{(t)}\Vert\leq\frac{\tilde{\epsilon}_{d}^{(t)}}{\sqrt{\beta}}$, and thus~\eqref{eq:tri_wc2} can be written as
        \begin{align} \label{eq:wc_w}
        &\Vert\bar{\mathbf w}_c^{(t+1)}-\bar{\mathbf w}^{(t+1)}\Vert\leq
        (1+\tilde{\eta}_t)\Vert\bar{\mathbf w}_c^{(t)}-\bar{\mathbf w}^{(t)}\Vert
                +\tilde{\eta}_t\sum\limits_{d=1}^N\varrho_{d}\Vert\bar{\mathbf w}_d^{(t)}-\bar{\mathbf w}^{(t)}\Vert
        \nonumber \\&
        +\frac{\tilde{\eta}_t}{\beta}\Vert\frac{1}{s_{c}}\sum\limits_{j\in\mathcal S_{c}}\mathbf n_{j}^{(t)}\Vert
        +\frac{\tilde{\eta}_t}{\beta}\Vert\sum\limits_{d=1}^N\varrho_{d}\frac{1}{s_{d}}\sum\limits_{j\in\mathcal S_{d}}\mathbf n_{j}^{(t)}\Vert
        +\frac{\tilde{\eta}_t}{\sqrt{\beta}}\left(\tilde{\epsilon}_{c}^{(t)}+\sum\limits_{d=1}^N\varrho_{d}\tilde{\epsilon}_{d}^{(t)}+\tilde{\delta}\right).
    \end{align} 
  Taking the weighted sum $\sum\limits_{c=1}^N\varrho_c$ from the both hand sides of the above inequality gives us
    \begin{align} 
        &\sum\limits_{c=1}^N\varrho_c\Vert\bar{\mathbf w}_c^{(t+1)}-\bar{\mathbf w}^{(t+1)}\Vert\leq
        (1+2\tilde{\eta}_t)\sum\limits_{c=1}^N\varrho_c\Vert\bar{\mathbf w}_c^{(t)}-\bar{\mathbf w}^{(t)}\Vert
        \nonumber \\&
        +\frac{\tilde{\eta}_t}{\beta}\sum\limits_{c=1}^N\varrho_c\Vert\frac{1}{s_{c}}\sum\limits_{j\in\mathcal S_{c}}\mathbf n_{j}^{(t)}\Vert
        +\frac{\tilde{\eta}_t}{\beta}\Vert\sum\limits_{d=1}^N\varrho_{d}\frac{1}{s_{d}}\sum\limits_{j\in\mathcal S_{d}}\mathbf n_{j}^{(t)}\Vert
        +\frac{\tilde{\eta}_t}{\sqrt{\beta}}\left(2\sum\limits_{d=1}^N\varrho_{d}\tilde{\epsilon}_{d}^{(t)}+\tilde{\delta}\right).
    \end{align}  
Multiplying the both hand sides of the above inequality by $\sqrt{\beta}$ followed by taking square and expectation, we get
\begin{align}
        \mathbb E\left[\sqrt{\beta}\sum\limits_{c=1}^N\varrho_c\Vert\bar{\mathbf w}_c^{(t+1)}-\bar{\mathbf w}^{(t+1)}\Vert\right]^2 \leq& 
         \mathbb E\bigg[\sqrt{\beta}(1+2\tilde{\eta}_t)\sum\limits_{c=1}^N\varrho_c\Vert\bar{\mathbf w}_c^{(t)}-\bar{\mathbf w}^{(t)}\Vert
        +\frac{\tilde{\eta}_t}{\sqrt{\beta}} \sum\limits_{c=1}^N\varrho_c\Vert\frac{1}{s_{c}}\sum\limits_{j\in\mathcal S_{c}}\mathbf n_{j}^{(t)}\Vert
        \nonumber \\&
        +\frac{\tilde{\eta}_t}{\sqrt{\beta}}\Vert\sum\limits_{d=1}^N\varrho_{d}\frac{1}{s_{d}}\sum\limits_{j\in\mathcal S_{d}}\mathbf n_{j}^{(t)}\Vert
        +\tilde{\eta}_t\left(2\sum\limits_{d=1}^N\varrho_{d}\tilde{\epsilon}_{d}^{(t)}+\tilde{\delta}\right)\bigg]^2.
    \end{align}
Taking the square roots from the both hand sides and using Fact \ref{fact:1} (See Appendix \ref{app:fact1}) yields:
\begin{align} \label{eq:fact_x2}
      \sqrt{\beta\mathbb E[(\sum\limits_{c=1}^N\varrho_c\Vert\bar{\mathbf w}_c^{(t+1)}-\bar{\mathbf w}^{(t+1)}\Vert)^2]} 
      \leq& 
         (1+2\tilde{\eta}_t)
\sqrt{\beta\mathbb E[(\sum\limits_{c=1}^N\varrho_c\Vert\bar{\mathbf w}_c^{(t)}-\bar{\mathbf w}^{(t)}\Vert)^2]}
        % \nonumber \\&
        +\tilde{\eta}_t\left(2\sum\limits_{d=1}^N\varrho_{d}\tilde{\epsilon}_{d}^{(t)}+\tilde{\delta}+2 \tilde{\sigma}\right).
    \end{align}
 
\textbf{(Part II) Finding the expression for $A^{(t)}$ from its connection with $\sqrt{\beta\mathbb E[(\sum\limits_{c=1}^N\varrho_c\Vert\bar{\mathbf w}_c^{(t)}-\bar{\mathbf w}^{(t)}\Vert)^2]}$}:
To bound the model dispersion across the clusters, we recursively expand of~\eqref{eq:fact_x2} and yield: 
\begin{align} \label{eq:yt}
    &\sqrt{\beta\mathbb E[(\sum\limits_{c=1}^N\varrho_c\Vert\bar{\mathbf w}_c^{(t)}-\bar{\mathbf w}^{(t)}\Vert)^2]} \leq
    (1+2\tilde{\eta}_{t-1})\sqrt{\beta\mathbb E[(\sum\limits_{c=1}^N\varrho_c\Vert\bar{\mathbf w}_c^{(t-1)}-\bar{\mathbf w}^{(t-1)}\Vert)^2]} 
        % \nonumber \\& 
        +\tilde{\eta}_{t-1}\left(2\sum\limits_{d=1}^N\varrho_{d}\tilde{\epsilon}_{d}^{(t-1)}+\tilde{\delta}+2\tilde{\sigma}\right)
    \nonumber \\&
    \leq 
    \sum_{\ell=t_{k-1}}^{t-1}\tilde{\eta}_\ell\prod_{j=\ell+1}^{t-1}(1+2\tilde{\eta}_j)\left(2\sum\limits_{d=1}^N\varrho_{d}\tilde{\epsilon}_{d}^{(0)}+\tilde{\delta}+2 \tilde{\sigma}\right)
    =\Sigma_t\left(2\sum\limits_{d=1}^N\varrho_{d}\tilde{\epsilon}_{d}^{(0)}+\tilde{\delta}+2 \tilde{\sigma}\right)
     =\Sigma_t\left(2\tilde{\epsilon}^{(0)}+\tilde{\delta}+2 \tilde{\sigma}\right)
    ,
    \end{align}
    where we define $
 \Sigma_{t}
 =\sum_{\ell=t_{k-1}}^{t-1}\tilde{\eta}_\ell\prod_{j=\ell+1}^{t-1}(1+2\tilde{\eta}_j).
    $ and $(a)$ comes from the fact that $\sum\limits_{d=1}^N\varrho_{d}\tilde{\epsilon}_{d}^{(0)}= \tilde{\epsilon}^{(0)}$.
Taking the square of the both hand sides, we get:
    \begin{align} \label{eq:Amid}
        &
 \beta\mathbb E\left[(\sum_{c=1}^N\varrho_c\Vert\bar{\mathbf w}_c^{(t)}-\bar{\mathbf w}^{(t)}\Vert)^2\right]
        \leq
[\Sigma_{t}]^2
\left[2\tilde{\epsilon}^{(0)}+\tilde{\delta}+2 \tilde{\sigma}\right]^2
% \nonumber \\&
\leq
12[\Sigma_{t}]^2
\left[\tilde{\sigma}^2+\tilde{\delta}^2+(\tilde{\epsilon}^{(0)})^2\right].
    \end{align}
    % Using the strong convexity of $F(.)$, we have
    % $\Vert\bar{\mathbf w}(t_{k-1})-\mathbf w^*\Vert^2
    % \leq \frac{2}{\tilde{\mu}\beta}
    % [F(\bar{\mathbf w}(t_{k-1}))-F(\mathbf w^*)]
    % $, using which in~\eqref{eq:Amid} yields:
    Applying the result from \eqref{eq:Amid}, we obtain the final result as  
    \begin{align} 
    \beta A^{(t)}
    &=\beta\mathbb E\left[\sum_{c=1}^N\varrho_c\Vert\bar{\mathbf w}_c^{(t)}-\bar{\mathbf w}^{(t)}\Vert^2\right]
    \leq
    \left(\varrho^{\mathsf{min}}\right)^{-1} \beta\mathbb E\left[(\sum_{c=1}^N\varrho_c\Vert\bar{\mathbf w}_c^{(t)}-\bar{\mathbf w}^{(t)}\Vert)^2\right]
    \nonumber \\&
    \leq
    12\left(\varrho^{\mathsf{min}}\right)^{-1}[\Sigma_{t}]^2
    \left[\tilde{\sigma}^2+\tilde{\delta}^2+(\tilde{\epsilon}^{(0)})^2\right]
    =
    12\left(\varrho^{\mathsf{min}}\right)^{-1}[\Sigma_{t}]^2
    \left[\frac{\sigma^2}{\beta}+\frac{\delta^2}{\beta}+\beta(\epsilon^{(0)})^2\right].
    \end{align}
    This concludes the proofs.
    
\end{proof}

\section{Proof of Theorem \ref{co1}} \label{app:thm1}
% \begin{corollary} \label{co1}
% Let $\bar{\mathbf w}^{(t)}
%         =\sum\limits_{c=1}^{N}\varrho_c^{(k)}\bar{\mathbf w}_c^{(t)}$ denote the global average of the local models, $\forall t$.
%         Under Assumption \ref{beta}, if $\eta_t=\frac{\gamma}{t+\alpha}\leq\vartheta/\beta$, $\epsilon_c^{(t)}=\eta_t\phi_c$, where $\sum_d\varrho_d\phi_d\leq\phi$ for some $\phi>0$, $\forall t$ and $\alpha\geq \max\{2/\vartheta^2,3\}$ and
%     $2/\mu\leq\gamma\leq\vartheta\alpha/\beta$, then the one-step behavior of $\bar{\mathbf w}^{(t)}$ under  {\tt TT-HF} for $t\in \mathcal{T}_k$ can be described as follows:
% % https://www.overleaf.com/project/5f7c9b5ce460a000011dc1e1
% \begin{align*} 
%       &\mathbb E\left[F(\hat{\mathbf w}^{(t+1)})-F(\mathbf w^*)\right]
%         \leq 
%         (1-\mu\eta_{t})\mathbb E\left[F(\hat{\mathbf w}^{(t)})-F(\mathbf w^*)\right]
%         % +\beta\eta_{t}^2\phi^2
%         \nonumber \\&
%         +\eta_{t}^2\beta^2
%         \bigg[ 
%          12C^2(t_{k-1}+\alpha-1)[1+\eta_{0}(2\beta\omega-\mu/2)]^2/(\mu\gamma^2) [F(\hat{\mathbf w}^{(t_{k-1})})-F(\mathbf w^*)]\\&
%         +6C^2\eta_0[1+\eta_{0}(2\beta\omega-\mu/2)]^2\tau_k^2(\beta\vartheta+\sigma)^2
%         +3C^2\eta_{0}(\sqrt{2}\beta\phi\eta_{0}+\sigma+\sqrt{2}\delta)^2
%         +\phi^2(\frac{1}{2\beta}+\eta_{0}/2)+\frac{\sigma^2}{2\beta}
%         \bigg].
% \end{align*}
% \end{corollary}

\begin{theorem} \label{co1}
        Under Assumptions \ref{beta},~\ref{assump:cons}, and~\ref{assump:SGD_noise}, upon using {\tt TT-HF} for ML model training, if $\eta_t \leq 1/\beta$, $\forall t$, the one-step behavior of $\hat{\mathbf w}^{(t)}$ can be described as follows:
% https://www.overleaf.com/project/5f7c9b5ce460a000011dc1e1
\begin{align*} 
       \mathbb E\left[F(\hat{\mathbf w}^{(t+1)})-F(\mathbf w^*)\right]
        \leq&
        (1-\mu\eta_{t})\mathbb E[F(\hat{\mathbf w}^{(t)})-F(\mathbf w^*)]
        % \nonumber \\&
        +\frac{\eta_{t}\beta^2}{2}A^{(t)}
        +\frac{1}{2}[\eta_{t}\beta^2(\epsilon^{(t)})^2+\eta_{t}^2\beta\sigma^2+\beta(\epsilon^{(t+1)})^2], ~t\in\mathcal{T}_k,
\end{align*}
where
\begin{align}
     A^{(t)}\triangleq\mathbb E\left[\sum\limits_{c=1}^N\varrho_{c}\Vert\bar{\mathbf w}_c^{(t)}-\bar{\mathbf w}^{(t)}\Vert_2^2\right].
\end{align}
\end{theorem}

\begin{proof}
% \add{(we omit the dependence on $k$ for conciseness)}.
% We define 
% the average of the local models for an arbitrary cluster $c$ as
% \nm{drop the dependence on k..}
% \begin{align}
%     \bar{\mathbf w}_c^{(t)}=\frac{1}{s_c^{(k)}}
%     \sum\limits_{j\in\mathcal{S}_c^{(k)}}\mathbf w_j^{(t)}.
% \end{align}
% and also define global average of the local models as
% \begin{align}
%     \bar{\mathbf w}^{(t)}&
%         =\sum_c\varrho_c^{(k)}\bar{\mathbf w}_c^{(t)}.
% \end{align}
Considering $t \in \mathcal T_k$, using \eqref{8}, \eqref{eq14}, the definition of $\bar{\mathbf{w}}$ given in Definition~\ref{modDisp}, and the fact that $\sum\limits_{i\in\mathcal{S}_c} \mathbf e_{i}^{{(t)}}=0$, $\forall t$, under Assumption~\ref{assump:cons},
the global average of the local models follows the following dynamics:
% \nm{define ${\mathbf n}_{j,t}=\widehat{\mathbf g}^{(t)}_{j}-\nabla F_j(\mathbf w_j^{(t)})$ (or something else) from the beginning so that you keep the math more concise}
\begin{align}\label{eq:GlobDyn1}
    \bar{\mathbf w}^{(t+1)}=
    \bar{\mathbf w}^{(t)}
    -\frac{\tilde{\eta}_{t}}{\beta}\sum\limits_{c=1}^N \varrho_c\frac{1}{s_c}\sum\limits_{j\in\mathcal{S}_c} 
    \nabla F_j(\mathbf w_j^{(t)})
              -\frac{\tilde{\eta}_{t}}{\beta}\sum\limits_{c=1}^N \varrho_c\frac{1}{s_c}\sum\limits_{j\in\mathcal{S}_c} 
             \mathbf n_j^{{(t)}},
\end{align} where ${\mathbf n}_{j}^{{(t)}}=\widehat{\mathbf g}^{(t)}_{j}-\nabla F_j(\mathbf w_j^{(t)})$.
On the other hand, the $\beta$-smoothness of the global function $F$ implies
\begin{align} 
    F(\bar{\mathbf w}^{(t+1)}) \leq  F(\bar{\mathbf w}^{(t)})&+\nabla F(\bar{\mathbf w}^{(t)})^\top(\bar{\mathbf w}^{(t+1)}-\bar{\mathbf w}^{(t)})+\frac{\beta}{2}\Big\Vert \bar{\mathbf w}^{(t+1)}-\bar{\mathbf w}^{(t)}\Big\Vert^2.
\end{align}
Replacing the result of~\eqref{eq:GlobDyn1} in the above inequality, taking the conditional expectation (conditioned on the knowledge of $\bar{\mathbf w}^{(t)}$) of the both hand sides, and using the fact that $\mathbb E_t[{\mathbf n}^{(t)}_{j}]=\bf 0$ yields:
\begin{align}
    &\mathbb E_{t}\left[F(\bar{\mathbf w}^{(t+1)})-F(\mathbf w^*)\right]
    \leq F(\bar{\mathbf w}^{(t)})-F(\mathbf w^*)
    -\frac{\tilde{\eta}_{t}}{\beta}\nabla F(\bar{\mathbf w}^{(t)})^\top \sum\limits_{c=1}^N\varrho_c\frac{1}{s_c}\sum\limits_{j\in\mathcal{S}_c}\nabla F_j(\mathbf w_j^{(t)})
    \nonumber \\&
    +\frac{\tilde{\eta}_{t}^2}{2\beta}\Big
    \Vert\sum\limits_{c=1}^N\varrho_c\frac{1}{s_c}\sum\limits_{j\in\mathcal{S}_c} \nabla F_j(\mathbf w_j^{(t)})
    \Big\Vert^2
    + \frac{\tilde{\eta}_{t}^2}{2\beta}\mathbb E_t\left[
    \Big
    \Vert\sum\limits_{c=1}^N\varrho_c\frac{1}{s_c}\sum\limits_{j\in\mathcal{S}_c}\mathbf n_j^{{(t)}}\Big\Vert^2
    \right].
\end{align}
% \add{where we used the fact that $\mathbb E_t[\widehat{\mathbf g}^{(t)}_{j}]=\nabla F_j(\mathbf w_j^{(t)})$}
% \nm{in the last step you used the fact that SGD is unbiased.. please explain}
Since $\mathbb E_t[\Vert\mathbf n_i^{{(t)}}\Vert_2^2]\leq \beta\tilde{\sigma}^2$, $\forall i$, we get
\begin{align}\label{eq:aveGlob1}
    \mathbb E_{t}\left[F(\bar{\mathbf w}^{(t+1)})-F(\mathbf w^*)\right]
        \leq& F(\bar{\mathbf w}^{(t)})-F(\mathbf w^*)
        \nonumber \\&
        -\frac{\tilde{\eta}_{t}}{\beta}\nabla F(\bar{\mathbf w}^{(t)})^\top \sum\limits_{c=1}^N\varrho_c\frac{1}{s_c}\sum\limits_{j\in\mathcal{S}_c}\nabla F_j(\mathbf w_j^{(t)})
        % \nonumber \\&
        +\frac{\tilde{\eta}_{t}^2}{2\beta}\Big
        \Vert\sum\limits_{c=1}^N\varrho_c\frac{1}{s_c}\sum\limits_{j\in\mathcal{S}_c} \nabla F_j(\mathbf w_j^{(t)})
        \Big\Vert^2
        +\frac{\tilde{\eta}_{t}^2\tilde{\sigma}^2}{2}.
\end{align}
Using Lemma \ref{lem1} (see Appendix \ref{app:PL-bound}), we further bound~\eqref{eq:aveGlob1} as follows:
\begin{align} \label{eq:E_t}
    &\mathbb E_{t}\left[F(\bar{\mathbf w}^{(t+1)})-F(\mathbf w^*)\right]
    \leq
        (1-\tilde{\mu}\tilde{\eta}_{t})(F(\bar{\mathbf w}^{(t)})-F(\mathbf w^*))
        \nonumber \\& 
       -\frac{\tilde{\eta}_{t}}{2\beta}(1-\tilde{\eta}_{t})\Big\Vert\sum\limits_{c=1}^N\varrho_c\frac{1}{s_c}\sum\limits_{j\in\mathcal{S}_c} \nabla F_j(\mathbf w_j^{(t)})\Big\Vert^2
       +\frac{\tilde{\eta}_{t}^2\tilde{\sigma}^2}{2}
        +\frac{\tilde{\eta}_{t}\beta}{2}\sum\limits_{c=1}^N\varrho_c \frac{1}{s_c}\sum\limits_{j\in\mathcal{S}_c}\Big\Vert\bar{\mathbf w}^{(t)}-\mathbf w_j^{(t)}\Big\Vert^2
    \nonumber \\&
    \leq
        (1-\tilde{\mu}\tilde{\eta}_{t})(F(\bar{\mathbf w}^{(t)})-F(\mathbf w^*))
        +\frac{\tilde{\eta}_{t}\beta}{2}\sum\limits_{c=1}^N\varrho_c \frac{1}{s_c}\sum\limits_{j\in\mathcal{S}_c}\Big\Vert\bar{\mathbf w}^{(t)}-\mathbf w_j^{(t)}\Big\Vert^2+\frac{\tilde{\eta}_{t}^2\tilde{\sigma}^2}{2},
\end{align}
where the last step follows from $\tilde{\eta}_t\leq 1$. 
To further bound the terms on the right hand side of~\eqref{eq:E_t}, we use the fact that
\begin{align}
        \Vert\mathbf w_i^{(t)}-\bar{\mathbf w}^{(t)}\Vert^2
        =
        \Vert\bar{\mathbf w}_c^{(t)}-\bar{\mathbf w}^{(t)}\Vert^2
        +\Vert \mathbf e_i^{{(t)}}\Vert^2
        +2[\bar{\mathbf w}_c^{(t)}-\bar{\mathbf w}^{(t)}]^\top \mathbf e_i^{{(t)}},
    \end{align} 
    % where $\mathbf e_i^{{(t_{k-1})}}=\mathbf 0$.
    which results in
    \begin{align} \label{eq:wi_w}
        \frac{1}{s_c}\sum\limits_{i\in \mathcal S_c}\Vert\mathbf w_i^{(t)}-\bar{\mathbf w}^{(t)}\Vert_2^2
        \leq
        \Vert\bar{\mathbf w}_c^{(t)}-\bar{\mathbf w}^{(t)}\Vert_2^2
        +\frac{(\tilde{\epsilon}_{c}^{(t)})^2}{\beta}.
    \end{align} 
    Replacing \eqref{eq:wi_w} in \eqref{eq:E_t}  and taking the unconditional expectation from the both hand sides of the resulting expression gives us
    \begin{align} \label{eq:ld}
        &\mathbb E\left[F(\bar{\mathbf w}^{(t+1)})-F(\mathbf w^*)\right]
        \leq
        (1-\tilde{\mu}\tilde{\eta}_{t})\mathbb E[F(\bar{\mathbf w}^{(t)})-F(\mathbf w^*)]
        \nonumber \\&
        +\frac{\tilde{\eta}_{t}\beta}{2}\sum\limits_{c=1}^N\varrho_c \left(\Vert\bar{\mathbf w}_c^{(t)}-\bar{\mathbf w}^{(t)}\Vert_2^2
        +\frac{(\tilde{\epsilon}_{c}^{(t)})^2}{\beta}\right)+\frac{\tilde{\eta}_{t}^2\tilde{\sigma}^2}{2}
        \nonumber \\&
        =
        (1-\tilde{\mu}\tilde{\eta}_{t})\mathbb E[F(\bar{\mathbf w}^{(t)})-F(\mathbf w^*)]
        +\frac{\tilde{\eta}_{t}\beta}{2}A^{(t)}
        +\frac{1}{2}[\tilde{\eta}_{t}(\tilde{\epsilon}^{(t)})^2+\tilde{\eta}_{t}^2\tilde{\sigma}^2],
    \end{align}
where  
\begin{align}
    A^{(t)}\triangleq\mathbb E\left[\sum\limits_{c=1}^N\varrho_c \Vert\bar{\mathbf w}_c^{(t)}-\bar{\mathbf w}^{(t)}\Vert^2\right].
\end{align}

By $\beta$-smoothness of $F(\cdot)$, we have
\begin{align} \label{eq:39}
    &F(\hat{\mathbf w}^{(t)})-F(\mathbf w^*) 
    \leq F(\bar{\mathbf w}^{(t)})-F(\mathbf w^*)
    +\nabla F(\bar{\mathbf w}^{(t)})^\top\Big(\hat{\mathbf w}^{(t)}-\bar{\mathbf w}^{(t)}\Big)+\frac{\beta}{2}\Vert\hat{\mathbf w}^{(t)}-\bar{\mathbf w}^{(t)}\Vert^2
    \nonumber \\&
    \leq F(\bar{\mathbf w}^{(t)})-F(\mathbf w^*)+\nabla F(\bar{\mathbf w}^{(t)})^\top\sum\limits_{c=1}^N\varrho_c \mathbf e_{s_c}^{{(t)}}
    +\frac{\beta}{2}\sum\limits_{c=1}^N\varrho_c\Vert \mathbf e_{s_c}^{{(t)}}\Vert^2.
\end{align}
Taking the expectation with respect to the device sampling from both hand sides of \eqref{eq:39}, since the sampling is conducted uniformly at random, we obtain
\begin{align} 
    \mathbb E_t\left[F(\hat{\mathbf w}^{(t)})-F(\mathbf w^*)\right]  
    \leq& F(\bar{\mathbf w}^{(t)})-F(\mathbf w^*)+\nabla F(\bar{\mathbf w}^{(t)})^\top\sum\limits_{c=1}^N\varrho_c \underbrace{\mathbb    E_t\left[\mathbf e_{s_c}^{{(t)}}\right]}_{=0}
    % \nonumber \\&
    +\frac{\beta}{2}\sum\limits_{c=1}^N\varrho_c\mathbb E_t\left[\Vert \mathbf e_{s_c}^{{(t)}}\Vert^2\right].
\end{align}
Taking the total expectation from both hand sides of the above inequality yields:
\begin{align} \label{eq:hat-bar}
    &\mathbb E\left[F(\hat{\mathbf w}^{(t)})-F(\mathbf w^*)\right]  
    \leq \mathbb E\left[F(\bar{\mathbf w}^{(t)})-F(\mathbf w^*)\right]
    +\frac{(\tilde{\epsilon}^{(t)})^2}{2}.
\end{align}
Replace \eqref{eq:ld} into \eqref{eq:hat-bar}, we have
\begin{align} \label{eq:-}
    \mathbb E\left[F(\hat{\mathbf w}^{(t+1)})-F(\mathbf w^*)\right]
        \leq&
        (1-\tilde{\mu}\tilde{\eta}_{t})\mathbb E[F(\bar{\mathbf w}^{(t)})-F(\mathbf w^*)]+\frac{\tilde{\eta}_{t}\beta}{2}A^{(t)}
        \nonumber \\&
        +\frac{1}{2}[\tilde{\eta}_{t}(\tilde{\epsilon}^{(t)})^2+\tilde{\eta}_{t}^2\tilde{\sigma}^2+(\tilde{\epsilon}^{(t+1)})^2].
\end{align}
On the other hands, using the strong convexity of $F(\cdot)$, we have
\begin{align} \label{eq:^}
    &F(\hat{\mathbf w}^{(t)})-F(\mathbf w^*) 
    \geq F(\bar{\mathbf w}^{(t)})-F(\mathbf w^*)
    +\nabla F(\bar{\mathbf w}^{(t)})^\top\Big(\hat{\mathbf w}^{(t)}-\bar{\mathbf w}^{(t)}\Big)+\frac{\mu}{2}\Vert\hat{\mathbf w}^{(t)}-\bar{\mathbf w}^{(t)}\Vert^2
    \nonumber \\&
    \geq F(\bar{\mathbf w}^{(t)})-F(\mathbf w^*)+\nabla F(\bar{\mathbf w}^{(t)})^\top\sum\limits_{c=1}^N\varrho_c \mathbf e_{s_c}^{{(t)}}.
    % \\&
    % \add{
    % \leq F(\hat{\mathbf w}^{(t)})-F(\mathbf w^*)
    % -\nabla F(\bar{\mathbf w}^{(t)})^\top\sum\limits_{c=1}^N\varrho_c \mathbf e_{s_c}^{{(t)}}
    % +\Vert\nabla F(\hat{\mathbf w}^{(t)})-\nabla F(\bar{\mathbf w}^{(t)})\Vert\Vert\sum\limits_{c=1}^N\varrho_c \mathbf e_{s_c}^{{(t)}}\Vert
    % +\frac{\beta}{2}\sum\limits_{c=1}^N\varrho_c\Vert \mathbf e_{s_c}^{{(t)}}\Vert^2
    % }
    %     \\&
    % \add{
    % \leq F(\hat{\mathbf w}^{(t)})-F(\mathbf w^*)
    % -\nabla F(\bar{\mathbf w}^{(t)})^\top\sum\limits_{c=1}^N\varrho_c \mathbf e_{s_c}^{{(t)}}
    % +\beta\Vert\sum\limits_{c=1}^N\varrho_c \mathbf e_{s_c}^{{(t)}}\Vert^2
    % +\frac{\beta}{2}\sum\limits_{c=1}^N\varrho_c\Vert \mathbf e_{s_c}^{{(t)}}\Vert^2
    % }
    %         \\&
    % \add{
    % \leq F(\hat{\mathbf w}^{(t)})-F(\mathbf w^*)
    % -\nabla F(\bar{\mathbf w}^{(t)})^\top\sum\limits_{c=1}^N\varrho_c \mathbf e_{s_c}^{{(t)}}
    % +\frac{3\beta}{2}\sum\limits_{c=1}^N\varrho_c\Vert \mathbf e_{s_c}^{{(t)}}\Vert^2
    % }
\end{align}
Taking the expectation with respect to the device sampling from the both hand sides of \eqref{eq:^}, since the sampling is conducted uniformly at random, we obtain
\begin{align} 
    &\mathbb E_t\left[F(\hat{\mathbf w}^{(t)})-F(\mathbf w^*)\right]  
    % \nonumber \\&
    \geq 
    F(\bar{\mathbf w}^{(t)})-F(\mathbf w^*)
    +\nabla F(\bar{\mathbf w}^{(t)})^\top\sum\limits_{c=1}^N\varrho_c \underbrace{\mathbb    E_t\left[\mathbf e_{s_c}^{{(t)}}\right]}_{=0}.
\end{align}
Taking the total expectation from both hand sides of the above inequality yields:
\begin{align} \label{eq:bar-hat}
    &\mathbb E\left[F(\hat{\mathbf w}^{(t)})-F(\mathbf w^*)\right]  
    \geq \mathbb E\left[F(\bar{\mathbf w}^{(t)})-F(\mathbf w^*)\right].
\end{align}
Finally, replacing \eqref{eq:bar-hat} into \eqref{eq:-}, we obtain
\begin{align}
    &\mathbb E\left[F(\hat{\mathbf w}^{(t+1)})-F(\mathbf w^*)\right]
        \leq
        (1-\tilde{\mu}\tilde{\eta}_{t})\mathbb E[F(\hat{\mathbf w}^{(t)})-F(\mathbf w^*)]
        % \nonumber \\&
        +\frac{\tilde{\eta}_{t}\beta}{2}A^{(t)}
        +\frac{1}{2}[\tilde{\eta}_{t}(\tilde{\epsilon}^{(t)})^2+\tilde{\eta}_{t}^2\tilde{\sigma}^2+(\tilde{\epsilon}^{(t+1)})^2]
        \nonumber
        \nonumber \\&
        =
        (1-\mu\eta_{t})\mathbb E[F(\hat{\mathbf w}^{(t)})-F(\mathbf w^*)]
        +\frac{\eta_{t}\beta^2}{2}A^{(t)}
        +\frac{1}{2}[\eta_{t}\beta^2(\epsilon^{(t)})^2+\eta_{t}^2\beta\sigma^2+\beta(\epsilon^{(t+1)})^2].
        % \nonumber \\&
        % \overset{(a)}{\leq}
        % (1-\tilde{\mu}\tilde{\eta}_{t})\mathbb E[F(\hat{\mathbf w}^{(t)})-F(\mathbf w^*)]
        % +\frac{\tilde{\eta}_{t}\beta}{2}A^{(t)}
        % +\frac{1}{2}[\tilde{\eta}_{t}^3\tilde{\phi}^2+\tilde{\eta}_{t}^2\tilde{\sigma}^2+\tilde{\eta}_{t+1}^2\tilde{\phi}^2]
        % \nonumber \\&
        % =
        % (1-\mu\eta_{t})\mathbb E[F(\hat{\mathbf w}^{(t)})-F(\mathbf w^*)]
        % +\frac{\eta_{t}\beta^2}{2}A^{(t)}
        % % \nonumber \\&
        % +\frac{1}{2}[\eta_{t}^3\phi^2\beta^2+\eta_{t}^2\beta\sigma^2+\eta_{t+1}^2\phi^2\beta],
\end{align}
% where $(a)$ comes from the fact that $\tilde{\epsilon}^{(t)}=\tilde{\eta}_t\tilde{\phi}$.
% $\tilde{\epsilon}^{(t)}=\tilde{\eta}_t\phi=\eta_t\phi\beta$. 
This concludes the proof.

\end{proof}

\section{Proof of Theorem \ref{thm:subLin}} \label{app:subLin}
% \addFL{
% \frank{Statement of the Theorem 2 needs modified:}
\begin{theorem} \label{thm:subLin}
Define $Z\triangleq
    \frac{1}{2}[\frac{\sigma^2}{\beta}+\frac{2\phi^2}{\beta}]
    +24\left(\varrho^{\mathsf{min}}\right)^{-1}\beta\gamma(\tau-1)\left(1+\frac{\tau-2}{\alpha}\right)
    \left(1+\frac{\tau-1}{\alpha-1}\right)^{4\beta\gamma}\left[\frac{\sigma^2}{\beta}+\frac{\phi^2}{\beta}+\frac{\delta^2}{\beta}\right]$. Also, assume $\gamma>1/\mu$ and $\alpha\geq\beta^2\gamma/\mu$. Upon using {\tt TT-HF} for ML model training under Assumptions \ref{beta},~\ref{assump:cons}, and~\ref{assump:SGD_noise}, if $\eta_t=\frac{\gamma}{t+\alpha}$ and $\epsilon^{(t)}=\eta_t\phi$, $\forall t$, we have: 
    % Using {\tt TT-HF} for ML model training, under Assumption \ref{beta}, if we set the step size as $\eta_t=\frac{\gamma}{t+\alpha}$, $\forall t$, and assuming that $\epsilon(t)=\eta_t\phi$, $\forall t$,
    % we have
    \begin{align} \label{eq:thm2_result-1-A}
        &\mathbb E\left[(F(\hat{\mathbf w}^{(t)})-F(\mathbf w^*))\right]\leq\frac{\nu}{t+\alpha}, ~~\forall t,
    \end{align}
    where $\nu \triangleq \max \left\{\frac{\beta^2\gamma^2Z}{\mu\gamma-1}, \alpha\left[F(\hat{\mathbf w}^{(0)})-F(\mathbf w^*)\right]\right\}$.
\end{theorem}
\begin{proof}

We carry out the proof by induction. We start by considering the first global aggregation, i.e., $k=1$. Note that the condition in~\eqref{eq:thm2_result-1-A} trivially holds at the beginning of this global aggregation $t=t_0=0$ since $
\nu\geq\alpha\left[F(\hat{\mathbf w}^{(0)})-F(\mathbf w^*)\right]$.
    Now, assume that
    \begin{align} 
        &\mathbb E\left[F(\hat{\mathbf w}^{(t_{k-1})})-F(\mathbf w^*)\right]  
        \leq \frac{\nu}{t_{k-1}+\alpha}
    \end{align}
for some $k\geq 1$. We prove that this implies 
\begin{align} \label{eq:thm2_result-1}
        &\mathbb E\left[F(\hat{\mathbf w}^{(t)})-F(\mathbf w^*)\right]  
        \leq \frac{\nu}{t+\alpha},\ \forall t\in \mathcal{T}_k,
    \end{align}
    and as a result $\mathbb E\left[F(\hat{\mathbf w}^{(t_{k})})-F(\mathbf w^*)\right]  
        \leq \frac{\nu}{t_{k}+\alpha}$.
    % so that the theorem follows by induction over $k$.
    To prove~\eqref{eq:thm2_result-1}, we use induction over $t\in \{t_{k-1}+1,\dots,t_k\}$.
    Clearly, the condition holds for $t=t_{k-1}$ from the induction hypothesis.
    Now, we assume that it also holds for some $t\in \{t_{k-1},\dots,t_k-1\}$, and aim to show that it holds at $t+1$.
    
    From the result of Theorem~\ref{co1}, considering $\tilde{\epsilon}^{(t)}=\tilde{\eta}_t\tilde{\phi}$, we get
    \begin{align} \label{eq:thm1_temp}
        &\mathbb E\left[F(\hat{\mathbf w}^{(t+1)})-F(\mathbf w^*)\right]
        \leq
        (1-\tilde{\mu}\tilde{\eta}_{t})\frac{\nu}{t+\alpha}
        +\frac{\tilde{\eta}_{t}\beta}{2}A^{(t)}
        +\frac{1}{2}[\tilde{\eta}_{t}^3\tilde{\phi}^2+\tilde{\eta}_{t}^2\tilde{\sigma}^2+\tilde{\eta}_{t+1}^2\tilde{\phi}^2].
    \end{align}
Using the induction hypothesis and the bound on $A^{(t)}$, we can further upper bound~\eqref{eq:thm1_temp} as
\begin{align} \label{eq:induction_main}
    &\mathbb E\left[F(\hat{\mathbf w}^{(t+1)})-F(\mathbf w^*)\right]
    \leq
    \left(1-\tilde{\mu}\tilde{\eta}_{t}\right)\frac{\nu}{t+\alpha}
    % \nonumber \\&
    +6\left(\varrho^{\mathsf{min}}\right)^{-1}\tilde{\eta}_t[\Sigma_{t}]^2
    \left[\tilde{\sigma}^2+(\tilde{\epsilon}^{(0)})^2+\tilde{\delta}^2\right]   
    +\frac{1}{2}[\tilde{\eta}_{t}^3\tilde{\phi}^2+\tilde{\eta}_{t}^2\tilde{\sigma}^2+\tilde{\eta}_{t+1}^2\tilde{\phi}^2].
\end{align}
Since $\tilde{\eta}_{t+1}\leq \tilde{\eta}_t$, $\tilde{\eta}_t\leq \tilde{\eta}_0\leq\tilde{\mu}\leq 1$ and $\tilde{\epsilon}^{(0)}=\tilde{\eta}_0\tilde{\phi}\leq\tilde{\phi}$, we further upper bound~\eqref{eq:induction_main} as
\begin{align} \label{eq:thm1_Sig}
    &\mathbb E\left[F(\hat{\mathbf w}^{(t+1)})-F(\mathbf w^*)\right]
    \leq
    \left(1-\tilde{\mu}\tilde{\eta}_{t}\right)\frac{\nu}{t+\alpha}
    % \nonumber \\&
    +6\left(\varrho^{\mathsf{min}}\right)^{-1}\tilde{\eta}_t\underbrace{[\Sigma_{t}]^2}_{(a)}
\left[\tilde{\sigma}^2+\tilde{\phi}^2+\tilde{\delta}^2\right]    
    + \frac{\tilde{\eta}_t^2}{2}[\tilde{\sigma}^2+2\tilde{\phi}^2].
\end{align}
To bound the instance of $[{\Sigma}_{t}]^2$, i.e., $(a)$ in~\eqref{eq:thm1_Sig}, we first use the fact that
% \begin{align*}
%     \lambda_+ =1-\tilde{\mu}/4+\sqrt{(1+\tilde{\mu}/4)^2+2\omega}\in[2,1+\sqrt{3}],
% \end{align*}
% which implies that
\begin{align}
    \Sigma_{t}
  =&\sum_{\ell=t_{k-1}}^{t-1}
  \eta_\ell
  \left[\prod_{j=\ell+1}^{t-1}(1+2\tilde{\eta}_j)\right]
%   \nonumber \\&
%   \nonumber \\&
  \leq
  \left[\prod_{j=t_{k-1}}^{t-1}(1+2\tilde{\eta}_j)\right]
\sum_{\ell=t_{k-1}}^{t-1}\frac{\tilde{\eta}_\ell}{1+\tilde{\eta}_\ell}.
\end{align}
Also, with the choice of step size $\tilde{\eta}_\ell=\frac{\tilde{\gamma}}{\ell+\alpha}$, we get
\begin{align} \label{eq:Sigma_1}
  \Sigma_{t}
  \leq
  \tilde{\gamma}\underbrace{\left[\prod_{j=t_{k-1}}^{t-1}\left(1+\frac{2\tilde{\gamma}}{j+\alpha}\right)\right]}_{(i)}
  \underbrace{\sum_{\ell=t_{k-1}}^{t-1}\frac{1}{\ell+\alpha+\tilde{\gamma}}}_{(ii)}.
\end{align}
To bound $(ii)$, since $\frac{1}{\ell+\alpha+\tilde{\gamma}}$ is a decreasing function with respect to $\ell$, we have
\begin{align}
    \sum_{\ell=t_{k-1}}^{t-1}\frac{1}{\ell+\alpha+\tilde{\gamma}}
    \leq
    \int_{t_{k-1}-1}^{t-1}\frac{1}{\ell+\alpha+\tilde{\gamma}}\mathrm d\ell
    =
    \ln\left(1+\frac{t-t_{k-1}}{t_{k-1}-1+\alpha+\tilde{\gamma}}\right),
\end{align}
where we used the fact that $\alpha>1-\tilde{\gamma}$ (implied by $\alpha>1$).

To bound $(i)$, we first rewrite it as follows: 
    \begin{align} \label{eq:(i)}
        \prod_{j=t_{k-1}}^{t-1}\left(1+\frac{2\tilde{\gamma}}{j+\alpha}\right)
        =
        e^{\sum_{j=t_{k-1}}^{t-1}\ln\big(1+\frac{2\tilde{\gamma}}{j+\alpha}\big)}
    \end{align}
To bound~\eqref{eq:(i)}, we use the fact that $\ln(1+\frac{2\tilde{\gamma}}{j+\alpha})$ is a decreasing function with respect to $j$, and $\alpha >1$, to get
\begin{align} \label{eq:ln(i)}
    &\sum_{j=t_{k-1}}^{t-1}\ln(1+\frac{2\tilde{\gamma}}{j+\alpha})
    \leq
    \int_{t_{k-1}-1}^{t-1}\ln(1+\frac{2\tilde{\gamma}}{j+\alpha})\mathrm dj
    \nonumber \\&
    \leq
    2\tilde{\gamma}\int_{t_{k-1}-1}^{t-1}\frac{1}{j+\alpha}\mathrm dj
    =
    2\tilde{\gamma} \ln\left(1+\frac{t-t_{k-1}}{t_{k-1}-1+\alpha}\right).
\end{align}
Considering ~\eqref{eq:(i)} and~\eqref{eq:ln(i)} together, we  bound $(i)$ as follows:
    \begin{align}
        \prod_{j=t_{k-1}}^{t-1}\left(1+\frac{2\tilde{\gamma}}{j+\alpha}\right)
        \leq
        \left(1+\frac{t-t_{k-1}}{t_{k-1}-1+\alpha}\right)^{2\tilde{\gamma}}.
    \end{align}

    Using the results obtained for bounding $(i)$ and $(ii)$ back in~\eqref{eq:Sigma_1}, we get:
    \begin{align} \label{eq:Sigma1_bound}
        \Sigma_{t}
        \leq
        \tilde{\gamma}\ln\left(1+\frac{t-t_{k-1}}{t_{k-1}-1+\alpha+\tilde{\gamma}}\right)
        \left(1+\frac{t-t_{k-1}}{t_{k-1}-1+\alpha}\right)^{2\tilde{\gamma}}.
    \end{align}
    Since $\ln(1+x)\leq\ln(1+x+2\sqrt{x})=\ln((1+\sqrt{x})^2)=2\ln(1+\sqrt{x})\leq 2\sqrt{x}$ for $x\geq 0$,
    we can further bound~\eqref{eq:Sigma1_bound} as follows:
    \begin{align} \label{eq:Sig}
        &\Sigma_{t}
        \leq
        2\tilde{\gamma}\sqrt{\frac{t-t_{k-1}}{t_{k-1}-1+\alpha+2\tilde{\gamma}}}
        \left(1+\frac{t-t_{k-1}}{t_{k-1}+\alpha-1}\right)^{2\tilde{\gamma}}
        \nonumber \\&
        \leq
        2\tilde{\gamma}\sqrt{\frac{t-t_{k-1}}{t_{k-1}+\alpha}}
        \left(1+\frac{t-t_{k-1}}{t_{k-1}+\alpha-1}\right)^{2\tilde{\gamma}},
    \end{align}
    where in the last inequality we used $\tilde{\gamma}\geq \frac{\tilde{\mu}}{\beta}\tilde{\gamma} >1$.

Taking the square from the both hand sides of~\eqref{eq:Sig} followed by multiplying the both hand sides with $[t+\alpha]$ gives us:
\begin{align}
    [t+\alpha] [\Sigma_{t}]^2
    &\leq
    4\tilde{\gamma}^2\frac{[t-t_{k-1}][t+\alpha]}{t_{k-1}+\alpha}
    \left(1+\frac{t-t_{k-1}}{t_{k-1}+\alpha-1}\right)^{4\tilde{\gamma}}
    \nonumber \\&
    \leq
    4\tilde{\gamma}^2(\tau-1)\left(1+\frac{\tau-2}{t_{k-1}+\alpha}\right)
    \left(1+\frac{\tau-1}{t_{k-1}+\alpha-1}\right)^{4\tilde{\gamma}}
    \nonumber \\&
    \leq
    4\tilde{\gamma}^2(\tau-1)\left(1+\frac{\tau-2}{\alpha}\right)
    \left(1+\frac{\tau-1}{\alpha-1}\right)^{4\tilde{\gamma}},
\end{align}
which implies
\begin{align} \label{eq:Sigma_2nd}
    [\Sigma_{t}]^2
    \leq
    4\tilde{\gamma}(\tau-1)\left(1+\frac{\tau-2}{\alpha}\right)
    \left(1+\frac{\tau-1}{\alpha-1}\right)^{4\tilde{\gamma}}\tilde{\eta}_t.
\end{align}
Replacing \eqref{eq:Sigma_2nd} into \eqref{eq:thm1_Sig}, we get
\begin{align} \label{eq:induce_form}
    \mathbb E[F(\hat{\mathbf w}^{(t+1)})-F(\mathbf w^*)]
    \leq
    \left(1-\tilde{\mu}\tilde{\eta}_t\right)\frac{\nu}{t+\alpha}
    +\tilde{\eta}_t^2Z,
\end{align}
where we have defined
% \begin{align}
%     Z_1\triangleq \frac{32\tilde{\gamma}}{\tilde{\mu}}(\tau-1)\left(1+\frac{\tau}{\alpha-1}\right)^{2}
%     \left(1+\frac{\tau-1}{\alpha-1}\right)^{6\tilde{\gamma}},
% \end{align}
% and
\begin{align}
    Z\triangleq
    \frac{1}{2}[\tilde{\sigma}^2+2\tilde{\phi}^2]
    +24\left(\varrho^{\mathsf{min}}\right)^{-1}\tilde{\gamma}(\tau-1)\left(1+\frac{\tau-2}{\alpha}\right)
    \left(1+\frac{\tau-1}{\alpha-1}\right)^{4\tilde{\gamma}}\left[\tilde{\sigma}^2+\tilde{\phi}^2+\tilde{\delta}^2\right].  
\end{align}

Now, from~\eqref{eq:induce_form}, to complete the induction, we aim to show that 
\begin{align}\label{eq:lastBeforeInd}
    \mathbb E[F(\hat{\mathbf w}^{(t+1)})-F(\mathbf w^*)]
    \leq
    \left(1-\tilde{\mu}\tilde{\eta}_t\right)\frac{\nu}{t+\alpha}
    +\tilde{\eta}_t^2Z
    \leq
    \frac{\nu}{t+1+\alpha}.
\end{align}
We transform the condition in~\eqref{eq:lastBeforeInd} through the set of following algebraic steps to an inequality condition on a convex function:
\begin{align}\label{eq:longtrasform}
    &
     \left(-\frac{\tilde{\mu}}{\tilde{\eta}_t^2} \right)\frac{\nu}{t+\alpha}
    +\frac{Z}{\tilde{\eta_t}}
    +\frac{\nu}{t+\alpha}\frac{1}{\tilde{\eta_t}^3}
    \leq\frac{\nu}{t+1+\alpha}\frac{1}{\tilde{\eta_t}^3}
    \nonumber \\&
    \Rightarrow
     \left(-\frac{\tilde{\mu}}{\tilde{\eta}_t} \right)\frac{\nu}{t+\alpha}\frac{1}{\tilde{\eta}_t}
    +\frac{Z}{\tilde{\eta_t}}
    +\left(\frac{\nu}{t+\alpha}-\frac{\nu}{t+1+\alpha}\right)\frac{1}{\tilde{\eta_t}^3}
    \leq0
    \nonumber \\&
    \Rightarrow
     \left(-\frac{\tilde{\mu}}{\tilde{\eta}_t} \right)\frac{\nu}{\tilde{\gamma}}
    +\frac{Z}{\tilde{\eta_t}}
    +\left(\frac{\nu}{t+\alpha}-\frac{\nu}{t+1+\alpha}\right)\frac{(t+\alpha)^3}{\tilde{\gamma}^3}
    \leq0
    \nonumber \\&
    \Rightarrow
     \left(-\frac{\tilde{\mu}}{\tilde{\eta}_t} \right)\frac{\nu}{\tilde{\gamma}}
    +\frac{Z}{\tilde{\eta_t}}
    +\frac{\nu}{(t+\alpha)(t+\alpha+1)}\frac{(t+\alpha)^3}{\tilde{\gamma}^3}
    \leq0
    \nonumber \\&
    \Rightarrow
     \left(-\frac{\tilde{\mu}}{\tilde{\eta}_t} \right)\frac{\nu}{\tilde{\gamma}}
    +\frac{Z}{\tilde{\eta_t}}
    +\frac{\nu}{t+\alpha+1}\frac{(t+\alpha)^2}{\tilde{\gamma}^3}
    \leq0
    \nonumber \\&
    \Rightarrow
      \tilde{\gamma}^2\left(-\frac{\tilde{\mu}}{\tilde{\eta}_t} \right)\nu
    +\frac{Z}{\tilde{\eta_t}} \tilde{\gamma}^3
    +\frac{(t+\alpha)^2}{t+\alpha+1}\nu
    \leq0
    \nonumber \\&
    \Rightarrow
      \tilde{\gamma}^2\left(-\frac{\tilde{\mu}}{\tilde{\eta}_t} \right)\nu
    +\frac{Z}{\tilde{\eta_t}} \tilde{\gamma}^3
    +\left(\frac{(t+\alpha+1)(t+\alpha-1)}{t+\alpha+1}\nu+\frac{\nu}{t+\alpha+1}\right)
    \leq 0,
\end{align}
where the last condition in~\eqref{eq:longtrasform} can be written as:
\begin{align}\label{eq:finTransform}
    \tilde{\gamma}^2\left(-\frac{\tilde{\mu}}{\tilde{\eta}_t}\right)\nu
    +\frac{Z}{\tilde{\eta}_t}\tilde{\gamma}^3
    +\nu[t+\alpha-1]
    +\frac{\nu}{t+1+\alpha}
    \leq 0.
\end{align}
Since the above condition needs to be satisfied $\forall t\geq 0$ and the expression on the left hand  side of the inequality is a convex function with respect to $t$  ($1/\eta_t$ is linear in $t$ and $\frac{1}{t+1+\alpha}$ is convex),
 it is sufficient to satisfy this condition for $t\to\infty $ and $t=0$. To obtain these limits, we first express~\eqref{eq:finTransform} as follows:
 \begin{align}\label{eq:fin2}
    -\tilde{\mu}\tilde{\gamma}(t+\alpha)\nu
    +Z \tilde{\gamma}^2(t+\alpha)
    +\nu[t+\alpha-1]
    +\frac{\nu}{t+1+\alpha}
    \leq 0.
 \end{align}
 Upon $t\to\infty$ considering the dominant terms yields
 \begin{align}\label{eq:condinf}
     &-\tilde{\gamma}\tilde{\mu}\nu t
    +Z \tilde{\gamma}^2t
    +\nu t
    \leq0
    \nonumber \\&
    \Rightarrow
    \left[1-\tilde{\gamma}\tilde{\mu}\right]\nu t
    +Z \tilde{\gamma}^2 t
    \leq 0.
 \end{align}
To satisfy~\eqref{eq:condinf}, the necessary condition is given by:
\begin{equation}
    \tilde{\mu}\tilde{\gamma}-1>0,
\end{equation}
 \begin{align} \label{eq:nu1}
     \nu \geq \frac{\tilde{\gamma}^2Z_2}{\tilde{\mu}\tilde{\gamma}-1}.
 \end{align}
 Also, upon $t\rightarrow 0$, from~\eqref{eq:fin2} we have
 \begin{align}
    &-\tilde{\mu}\tilde{\gamma}\alpha\nu
    +Z \tilde{\gamma}^2\alpha
    +\nu[\alpha-1]
    +\frac{\nu}{1+\alpha}\leq0
    \nonumber \\&
    \Rightarrow
    \nu\left(\alpha(\tilde{\mu}\tilde{\gamma}-1)+\frac{\alpha}{1+\alpha}\right)
    \geq
  \tilde{\gamma}^2 Z\alpha,
 \end{align}
which implies:
\begin{align} \label{eq:nu2}
    \nu \geq \frac{\tilde{\gamma}^2Z_2}{(\tilde{\mu}\tilde{\gamma}-1)+\frac{1}{1+\alpha}}.
\end{align}
Compared with~\eqref{eq:nu2}, ~\eqref{eq:nu1} imposes a stricter condition on $\nu$. Therefore, when 
\begin{align}
    \nu \geq Z\frac{\beta^2\gamma^2}{\mu\gamma-1},
\end{align}
we complete the induction and thus the proof.

\end{proof}

\section{Proof of Lemma \ref{lem:consensus-error}} \label{app:consensus-error}

\begin{lemma} \label{lem:consensus-error}
If the consensus matrix $\mathbf{V}_{{c}}$ satisfies Assumption~\ref{assump:cons}, then after performing $\Gamma^{(t)}_{{c}}$ rounds of consensus in cluster $\mathcal{S}_c$, the consensus error $\mathbf e_i^{(t)}$ is upper-bounded as:
\begin{equation} 
   \Vert \mathbf e_i^{(t)} \Vert  \hspace{-.5mm}  \leq (\lambda_{{c}})^{\Gamma^{(t)}_{{c}}}
\sqrt{s_c}\underbrace{\max_{j,j'\in\mathcal S_c}\Vert\tilde{\mathbf w}_j^{(t)}-\tilde{\mathbf w}_{j'}^{(t)}\Vert}_{\triangleq \Upsilon^{(t)}_{{c}}},
~ \forall i\in \mathcal{S}_c,
\end{equation}
where each $\lambda_{{c}}$ is a constant such that $1 > \lambda_{{c}} \geq \rho \left(\mathbf{V}_{{c}}-\frac{\textbf{1} \textbf{1}^\top}{s_c} \right)$.
\end{lemma} 
\begin{proof}
    The evolution of the devices' parameters can be described by~\eqref{eq:ConsCenter} as:
  \begin{equation}
      \mathbf{W}^{(t)}_{{c}}= \left(\mathbf{V}_{{c}}\right)^{\Gamma^{(t)}_{{c}}} \widetilde{\mathbf{W}}^{(t)}_{{c}}, ~t\in\mathcal T_k,
  \end{equation}
  where 
  \begin{align}
      \mathbf{W}^{(t)}_{{c}}=\left[\mathbf{w}^{(t)}_{{c_1}},\mathbf{w}^{(t)}_{{c_2}},\dots,\mathbf{w}^{(t)}_{{s_c}}\right]^\top
  \end{align}
  and 
  \begin{align}
      \widetilde{\mathbf{W}}^{(t)}_{{c}}=\left[\tilde{\mathbf{w}}^{(t)}_{{c_1}},\tilde{\mathbf{w}}^{(t)}_{{c_2}},\dots,\tilde{\mathbf{w}}^{(t)}_{{s_c}}\right]^\top.
  \end{align}
%   with $(\cdot)_q$ denote the $q$-th element of the vector. 
  
  Let matrix $\overline{\mathbf{W}}^{(t)}_{{c}}$ denote be the matrix with rows given by the average model parameters across the cluster, it can be represented as:
  \begin{align}
      \overline{\mathbf{W}}^{(t)}_{{c}}=\frac{\mathbf 1_{s_c} \mathbf 1_{s_c}^\top\widetilde{\mathbf{W}}^{(t)}_{{c}}}{s_c}.
  \end{align}
  We then define $\mathbf E^{(t)}_{{c}}$ as 
  \begin{align}
      \mathbf E^{(t)}_{{c}}= {\mathbf{W}}^{(t)}_{{c}}-\overline{\mathbf{W}}^{(t)}_{{c}}
      =[\left(\mathbf{V}_{{c}}\right)^{\Gamma^{(t)}_{{c}}}-\mathbf 1\mathbf 1^\top/s_c] [\widetilde{\mathbf{W}}^{(t)}_{{c}}-\overline{\mathbf{W}}^{(t)}_{{c}}]
      ,
  \end{align}
  so that $[\mathbf E^{(t)}_{{c}}]_{i,:}=\mathbf e_{i}^{(t)}$, where $[\mathbf E^{(t)}_{{c}}]_{i,:}$ is the $i$th row of $\mathbf E^{(t)}_{{c}}$.
  
   Therefore, using Assumption \ref{assump:cons}, we can bound the consensus error as
  \begin{align} \label{eq:consensus-bound-1}
      &\Vert\mathbf e_{i}^{(t)}\Vert^2\leq
      \mathrm{trace}((\mathbf E^{(t)}_{{c}})^{\top}\mathbf E^{(t)}_{{c}})
      \\&\nonumber
      =
      \mathrm{trace}\Big(
      [\widetilde{\mathbf{W}}^{(t)}_{{c}}{-}\overline{\mathbf{W}}^{(t)}_{{c}}]^\top
      [\left(\mathbf{V}_{{c}}\right)^{\Gamma^{(t)}_{{c}}}{-}\mathbf 1\mathbf 1^\top/s_c]^2 [\widetilde{\mathbf{W}}^{(t)}_{{c}}{-}\overline{\mathbf{W}}^{(t)}_{{c}}]
        \Big)
        \\&
        \leq (\lambda_{{c}})^{2\Gamma^{(t)}_{{c}}}
        \sum_{j=1}^{s_c}\Vert\tilde{\mathbf w}_j^{(t)}-\bar{\mathbf w}_c^{(t)}\Vert^2
        \nonumber
        \\&
        \leq (\lambda_{{c}})^{2\Gamma^{(t)}_{{c}}}
        \frac{1}{s_c}\sum_{j,j'=1}^{s_c}\Vert\tilde{\mathbf w}_j^{(t)}-\tilde{\mathbf w}_{j'}^{(t)}\Vert^2
        \nonumber
        \\&
        \nonumber
        \leq (\lambda_{{c}})^{2\Gamma^{(t)}_{{c}}}
        s_c\max_{j,j'\in\mathcal S_c}\Vert\tilde{\mathbf w}_j^{(t)}-\tilde{\mathbf w}_{j'}^{(t)}\Vert^2.
    \end{align}
  The result of the Lemma directly follows.
%   \begin{align} \label{eq:r-definition}
\end{proof}
% }
\section{Proof of Lemma \ref{lem1}} \label{app:PL-bound}
 
\begin{lemma} \label{lem1}
Under Assumption \ref{beta}, we have
%\nm{$\eta_t$ not $\eta_k$!}
    \begin{align*}
        &-\frac{\tilde{\eta}_{t}}{\beta}\nabla F(\bar{\mathbf w}^{(t)})^\top \sum\limits_{c=1}^N\varrho_c\frac{1}{s_c}\sum\limits_{j\in\mathcal{S}_c} \nabla F_j(\mathbf w_j^{(t)})
        \leq
        -\tilde{\mu}\tilde{\eta}_{t}(F(\bar{\mathbf w}^{(t)})-F(\mathbf w^*))
        \nonumber \\&
        -\frac{\tilde{\eta}_{t}}{2\beta}\Big\Vert\sum\limits_{c=1}^N\varrho_c\frac{1}{s_c}\sum\limits_{j\in\mathcal{S}_c} \nabla F_j(\mathbf w_j^{(t)})\Big\Vert^2
        +\frac{\tilde{\eta}_{t}\beta}{2}\sum\limits_{c=1}^N\varrho_c \frac{1}{s_c}\sum\limits_{j\in\mathcal{S}_c}\Big\Vert\bar{\mathbf w}^{(t)}-\mathbf w_j^{(t)}\Big\Vert^2.
    \end{align*}
\end{lemma}

\begin{proof}  Since $-2 \mathbf a^\top \mathbf b = -\Vert \mathbf a\Vert^2-\Vert \mathbf b \Vert^2 + \Vert \mathbf a- \mathbf b\Vert^2$ holds for any two vectors $\mathbf a$ and $\mathbf b$ with real elements, we have
    \begin{align} \label{14}
        &-\frac{\tilde{\eta}_{t}}{\beta}\nabla F(\bar{\mathbf w}^{(t)})^\top
        \sum\limits_{c=1}^N\varrho_c\frac{1}{s_c}\sum\limits_{j\in\mathcal{S}_c} \nabla F_j(\mathbf w_j^{(t)})
        \nonumber \\&
        =\frac{\tilde{\eta}_{t}}{2\beta}\bigg[-\Big\Vert\nabla F(\bar{\mathbf w}^{(t)})\Big\Vert^2
        -\Big\Vert\sum\limits_{c=1}^N\varrho_c\frac{1}{s_c}\sum\limits_{j\in\mathcal{S}_c} \nabla F_j(\mathbf w_j^{(t)})\Big\Vert^2
        \nonumber \\&
        +\Big\Vert\nabla F(\bar{\mathbf w}^{(t)})-\sum\limits_{c=1}^N\varrho_c\frac{1}{s_c}\sum\limits_{j\in\mathcal{S}_c} \nabla F_j(\mathbf w_j^{(t)})\Big\Vert^2\bigg].
    \end{align}
    Since $\Vert\cdot\Vert^2$ is a convex function, using Jenson's inequality, we get: $\Vert \sum_{i=1}^j {c_j} \mathbf a_j  \Vert^2 \leq \sum_{i=1}^j {c_j} \Vert  \mathbf a_j  \Vert^2 $, where $\sum_{i=1}^j {c_j} =1$. Using this fact in~\eqref{14} yields
    \begin{align} \label{20}
        &-\frac{\tilde{\eta}_{t}}{\beta}\nabla F(\bar{\mathbf w}^{(t)})^\top \sum\limits_{c=1}^N\varrho_c\frac{1}{s_c}\sum\limits_{j\in\mathcal{S}_c} \nabla F_j(\mathbf w_j^{(t)})
        \nonumber \\&
        \leq\frac{\tilde{\eta}_{t}}{2\beta}\bigg[-\Big\Vert\nabla F(\bar{\mathbf w}^{(t)})\Big\Vert^2
        -\Big\Vert\sum\limits_{c=1}^N\varrho_c\frac{1}{s_c}\sum\limits_{j\in\mathcal{S}_c} \nabla F_j(\mathbf w_j^{(t)})\Big\Vert^2
        \nonumber \\&
        +\sum\limits_{c=1}^N\varrho_c\frac{1}{s_c}\sum\limits_{j\in\mathcal{S}_c}\Big\Vert\nabla F_j(\bar{\mathbf w}^{(t)})-\nabla F_j(\mathbf w_j^{(t)})\Big\Vert^2\bigg].
    \end{align}
    
   Using $\mu$-strong convexity of $F(.)$, we get: $\Big\Vert\nabla F(\bar{\mathbf w}^{(t)})\Big\Vert^2 \geq 2\tilde{\mu}\beta (F(\bar{\mathbf w}^{(t)})-F(\mathbf w^*))$. Also, using $\beta$-smoothness of $F_j(\cdot)$ we get $\Big\Vert\nabla F_j(\bar{\mathbf w}^{(t)})-\nabla F_j(\mathbf w_j^{(t)})\Big\Vert^2 \leq \beta^2 \Vert\bar{\mathbf w}^{(t)}-\mathbf w_j^{(t)}\Big\Vert^2$, $\forall j$. Using these facts in~\eqref{20} yields:
    \begin{align}
        &-\frac{\tilde{\eta}_{t}}{\beta}\nabla F(\bar{\mathbf w}^{(t)})^\top \sum\limits_{c=1}^N\varrho_c\frac{1}{s_c}\sum\limits_{j\in\mathcal{S}_c} \nabla F_j(\mathbf w_j^{(t)})
        \leq
        -\tilde{\mu}\tilde{\eta}_{t}(F(\bar{\mathbf w}^{(t)})-F(\mathbf w^*))
        \nonumber \\&
        -\frac{\tilde{\eta}_{t}}{2\beta}\Big\Vert\sum\limits_{c=1}^N\varrho_c\frac{1}{s_c}\sum\limits_{j\in\mathcal{S}_c} \nabla F_j(\mathbf w_j^{(t)})\Big\Vert^2
        +\frac{\tilde{\eta}_{t}\beta}{2}\sum\limits_{c=1}^N\varrho_c \frac{1}{s_c}\sum\limits_{j\in\mathcal{S}_c}\Big\Vert\bar{\mathbf w}^{(t)}-\mathbf w_j^{(t)}\Big\Vert^2,
    \end{align}
    which concludes the proof.
\end{proof}

\section{Proof of Fact \ref{fact:1}} \label{app:fact1}
\begin{fact}\label{fact:1} For an arbitrary set of $n$ random variables $x_1,\cdots,x_n$, we have: 
 \begin{equation}
     \sqrt{\mathbb E\left[\left(\sum\limits_{i=1}^{n} x_i\right)^2\right]}\leq \sum\limits_{i=1}^{n} \sqrt{\mathbb E[x_i^2]}.
 \end{equation}
\end{fact}
\begin{proof} The proof can be carried out through the following set of algebraic manipulations:
    % Using the following result of Cauchy-Schwarz Inequality
    % \begin{align}
    %     \mathbb E[XY] \leq \sqrt{\mathbb E[X^2]\mathbb E[ Y^2]},
    % \end{align}
    % we obtain
    \begin{align}
        &\sqrt{\mathbb E\left[\left(\sum\limits_{i=1}^{n} x_i\right)^2\right]}
        =
        \sqrt{\sum\limits_{i=1}^{n} \mathbb E[x_i^2] +\sum\limits_{i=1}^{n}\sum\limits_{j=1,j\neq i}^{n}\mathbb E [x_i x_j]}
        \nonumber \\&
        \overset{(a)}{\leq}
        \sqrt{\sum\limits_{i=1}^{n} \mathbb E[x_i^2] +\sum\limits_{i=1}^{n}\sum\limits_{j=1,j\neq i}^{n}\sqrt{\mathbb E [x_i^2] \mathbb E[x_j^2]]}}
        % \nonumber \\&
        =\sqrt{\Big(\sum\limits_{i=1}^{n} \sqrt{\mathbb E[x_i^2]}\Big)^2}
        =
        \sum\limits_{i=1}^{n} \sqrt{\mathbb E[x_i^2]},
    \end{align}
    where $(a)$ is due to the fact that $\mathbb E[XY] \leq \sqrt{\mathbb E[X^2]\mathbb E[ Y^2]}$ resulted from Cauchy-Schwarz inequality.
\end{proof}
% \begin{proof}
%     \begin{align}
%         &\sqrt{\mathbb E[X+Y]^2}=\sqrt{\mathbb E[X^2+2XY+Y^2]}
%         =\sqrt{\mathbb E[X^2]+2\mathbb E[XY]+\mathbb[Y^2]]}
%         \nonumber \\&
%         \leq
%         \sqrt{\mathbb E[X^2]+\mathbb[Y^2]]+2\sqrt{\mathbb E [X^2] \mathbb E[Y^2]]}}
%         =\sqrt{\mathbb E[X^2]}+\sqrt{\mathbb E[Y^2]}.
%     \end{align}
% \end{proof}

\endgroup

\end{document}